%% file: main.tex
\documentclass{article}
\usepackage{ijcai26}

\usepackage{times}
\usepackage{soul}
\usepackage{url}
\usepackage[hidelinks]{hyperref}
\usepackage[utf8]{inputenc}
\usepackage[small]{caption}
\usepackage{graphicx}
\usepackage{amsmath}
\usepackage{multirow}
\usepackage{amsthm}
\usepackage{booktabs}
\usepackage{subcaption}
\usepackage{algorithm}
\usepackage{algorithmic}
\usepackage[switch]{lineno}

\usepackage{amssymb}
\usepackage{mathtools}

\usepackage[capitalize,noabbrev]{cleveref}

\usepackage{listings}
\usepackage{xcolor}

\definecolor{samplerbg}{RGB}{240, 248, 255}   
\definecolor{surrogatebg}{RGB}{240, 255, 240} 
\definecolor{contextbg}{RGB}{255, 250, 240}   
\definecolor{framecolor}{RGB}{100, 100, 100}

\lstdefinestyle{promptstyle}{
    basicstyle=\ttfamily\scriptsize,
    breaklines=true,
    frame=single,
    rulecolor=\color{framecolor},
    framesep=5pt,
    aboveskip=10pt,
    belowskip=10pt,
    captionpos=b,
    keepspaces=true,
    showstringspaces=false
}

\lstdefinestyle{sampler}{
    style=promptstyle,
    backgroundcolor=\color{samplerbg},
    title={\bfseries \color{blue!40!black} Candidate Sampler Prompt}
}

\lstdefinestyle{surrogate}{
    style=promptstyle,
    backgroundcolor=\color{surrogatebg},
    title={\bfseries \color{green!40!black} Surrogate Model Prompt}
}

\lstdefinestyle{context}{
    style=promptstyle,
    backgroundcolor=\color{contextbg},
    title={\bfseries \color{orange!40!black} Domain Context Injection}
}

\theoremstyle{plain}
\newtheorem{theorem}{Theorem}[section]

\newtheorem{lemma}[theorem]{Lemma}
\newtheorem{corollary}[theorem]{Corollary}
\theoremstyle{definition}
\newtheorem{definition}[theorem]{Definition}
\newtheorem{assumption}[theorem]{Assumption}
\theoremstyle{remark}

\newcommand{\R}{\mathbb{R}}
\newcommand{\X}{\mathcal{X}}
\newcommand{\Xstar}{\mathcal{X}^*}

\title{Multi-Objective Hierarchical Optimization with Large Language Models}

\author{
Andrej Schwanke$^1$
\and
Lyubomir Ivanov$^1$\and
David Salinas$^2$\and
Frank Hutter$^{3,2,1}$\And
Arber Zela$^{2}$\\
\affiliations
$^1$University of Freiburg, Germany\\
$^2$ELLIS Institute Tübingen, Germany\\
$^3$Prior Labs\\
\emails
andrejschwanke19@gmail.com \quad\quad arber.zela@tue.ellis.eu
}

\begin{document}

\maketitle

\begin{abstract}
Despite their widespread adoption in various domains, especially due to their powerful reasoning capabilities, Large Language Models (LLMs) are not the off-the-shelf choice to drive multi-objective optimization yet. Conventional strategies rank high in benchmarks due to their intrinsic capabilities to handle numerical inputs and careful modelling choices that balance exploration and Pareto-front exploitation, as well as handle multiple (conflicting) objectives. In this paper, we close this gap by leveraging LLMs as surrogate models and candidate samplers inside a structured hierarchical search strategy. By adaptively partitioning the input space into disjoint hyperrectangular regions and ranking them with a composite score function, we restrict the generative process of the LLM to specific, high-potential sub-spaces, hence making the problem easier to solve as the LLM doesn't have to reason about the global structure of the problem, but only locally instead. We show that under standard regularity assumptions, our algorithm generates candidate solutions that converge to the true Pareto set in Hausdorff distance. Empirically, it consistently outperforms the global LLM-based multi-objective optimizer and is on par with standard evolutionary and Bayesian optimization algorithm on synthetic and real-world benchmarks.
\end{abstract}

\section{Introduction}\label{sec:intro}

Optimizing multiple unknown functions with potentially conflicting objectives remains a challenging yet important problem arising in many areas~\cite{chinchuluun_survey,coello_moo_survey}, including machine learning~\cite{sener18}, engineering~\cite{Marler2004SurveyOM}, material science~\cite{gopakumar18} and robotics~\cite{tesch13}. For such multi-objective optimization (MOO) problems, finding a single solution that optimizes all objectives simultaneously is typically impossible, hence, we are interested in a set of Pareto optimal solutions~\cite{deb2001multi}.

In practice, for many MOO problems, we do not have access to gradient information and at least one of the objectives is expensive to evaluate~\cite{jones98}. For example, synthesizing molecules often requires costly wet-lab experiments or computationally expensive simulations, requiring days or potentially weeks for obtaining a single function evaluation~\cite{shields18}. Classical approaches such as Multi-Objective Evolutionary Algorithms (MOEAs)~\cite{deb2001multi,deb2002fast,deb2014evolutionary} and Bayesian Optimization (MOBO)~\cite{daulton2021parallel,daulton22} are standard and effective choices. However, they typically require assumptions on function properties and lack mechanisms to directly integrate domain knowledge into the optimization process.

\begin{figure}
    \centering
    \includegraphics[width=0.99\linewidth, height=0.21\textheight]{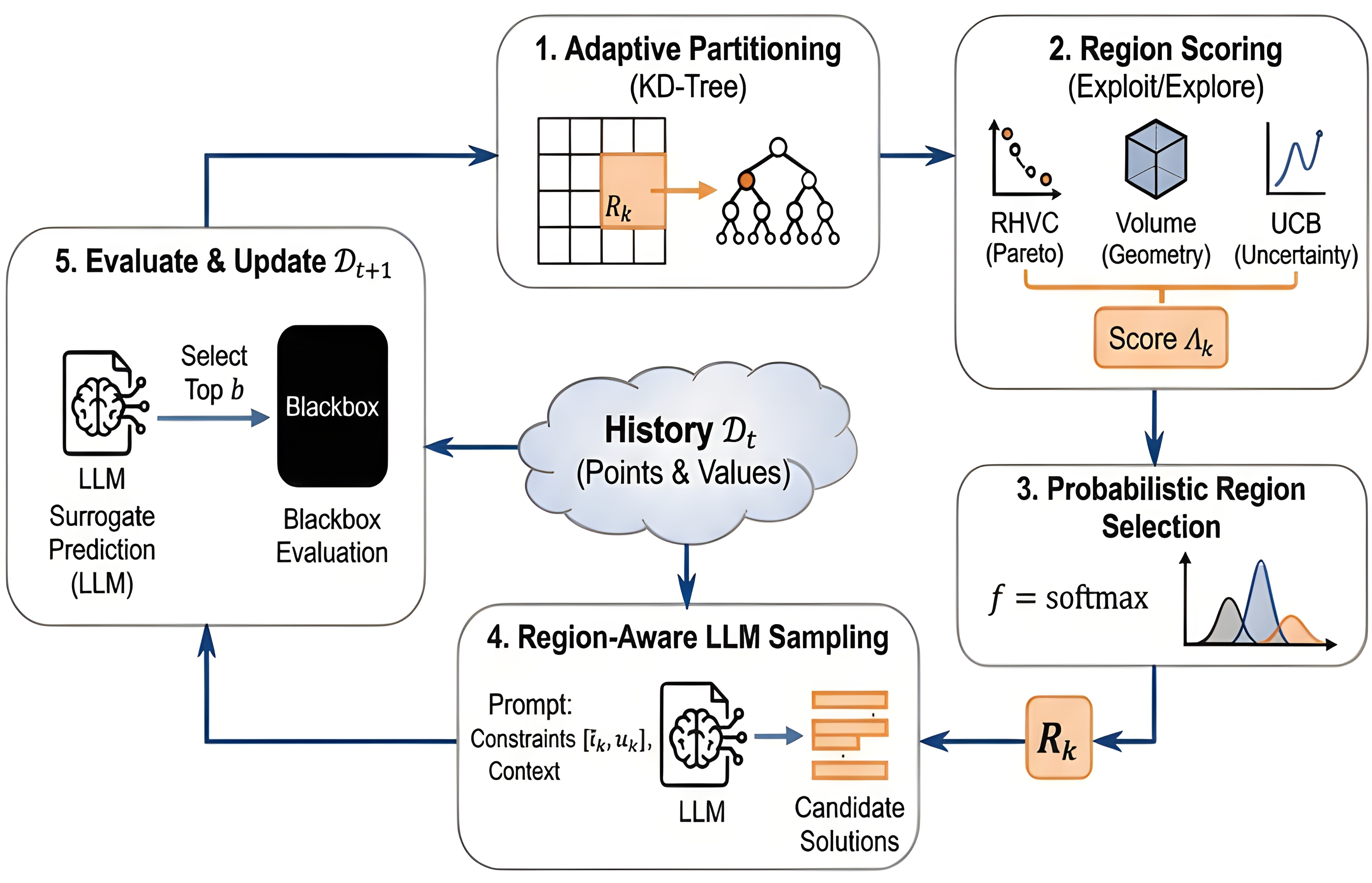}
    \caption{Illustration of MOHOLLM's optimization pipeline.}
    \label{fig:pipeline}
    \vspace{-2mm}
\end{figure}

In the past few years, Large Language Models (LLMs) \cite{touvron2023,openai2023gpt4} and their in-context learning (ICL) capabilities ~\cite{brown2020language,krishnamurthy24} have created substantial advancements across various fields. More recently, we have seen an adoption of LLMs in optimization too~\cite{liu2024large,agarwal2025searching,zhang2023using,yang2024large}, where LLMs are used to guide the search by generating candidate solutions by encoding task information and specifications inside prompts. 
However, recent findings suggest that such LLM-based optimizers may exhibit high sampling bias and sensitivity to prompt design \cite{liu2024large,schwanke2025improving}. Moreover, their extension to MOO involves additional complexities arising from the necessity to balance the multiple objectives.

In this paper, we introduce \textbf{MOHOLLM} (\textbf{M}ulti-\textbf{O}bjective \textbf{H}ierarchical \textbf{O}ptimization with \textbf{LLM}s), a framework that mitigates these issues through adaptive search space partitioning and region-aware prompting. Similarly to the recent work on LLM-based single-objective optimization~\cite{schwanke2025improving}, our method decomposes the search space into smaller sub-regions, balances exploration–exploitation through a composite scoring function, and prompts the LLM to generate candidates in the selected promising sub-regions. However, differently from \cite{schwanke2025improving}, we define a novel regional Hypervolume contribution score which is used both as exploitation factor and for the uncertainty component computation. We refer to Figure~\ref{fig:pipeline} for an illustration of the pipeline.

To summarize, our key contributions are:
\begin{itemize}
\setlength{\leftmargin}{0pt}
\setlength{\itemindent}{0pt}
\item We introduce MOHOLLM, a hierarchical LLM-based MOO method that enforces structured exploration via adaptive input space partitioning and a composite utility function combining regional Hypervolume contribution, geometric, and uncertainty-driven criteria. (Section~\ref{sec:method})
\item Under standard regularity and LLM behavior assumptions, we prove that MOHOLLM ensures asymptotic coverage of the true Pareto set. (Section~\ref{sec:theory})
\item Experiments on 15 synthetic and real-world benchmarks show that MOHOLLM is superior to global LLM-based optimizers and is on par with state-of-the-art MOBO and MOEAs. (Section~\ref{sec:experiments})
\end{itemize}

Together, these contributions enable a more reliable and principled integration of LLMs into MOO frameworks. We release our code in \url{https://github.com/arberzela/mohollm.git}.

\section{Background}\label{sec:background}

\subsection{Preliminaries}\label{sec:preliminaries}

Our goal is to solve the problem of global multi-objective optimization for black-box functions where the underlying gradients are typically unavailable and function evaluations are computationally expensive. Let $\mathbf{F}: \mathcal{X} \to \mathbb{R}^M$ be the vector-valued objective function defined over a compact hyperrectangular domain $\mathcal{X} \subset \mathbb{R}^d$:
\begin{equation}
\min_{x \in \mathcal{X}} \mathbf{F}(x) = (f_1(x), f_2(x), \ldots, f_M(x))^\top.
\end{equation}
We assume that the objectives are conflicting, meaning no single solution simultaneously optimizes all functions. Therefore, we are interested in a set of non-dominated solutions: the Pareto set $\mathcal{X}^\star$. 

\begin{definition}{(\textit{Pareto Dominance})}
Given two solutions $x, x' \in \mathcal{X}$, we say that $x$ \emph{Pareto dominates} $x'$ (denoted $x \prec x'$) if:
\begin{equation*}
    \forall i \in \{1,\ldots,M\}:\ f_i(x) \le f_i(x')\ \text{and}\ \exists j:\ f_j(x) < f_j(x').
\end{equation*}
\end{definition}

\begin{definition}{(\textit{Pareto Set})}
The Pareto set (also called Pareto optimal set) is the set of all non-dominated solutions in the search space: $\mathcal{X}^\star = \{\, x \in \mathcal{X} \mid \nexists\, x' \in \mathcal{X} \text{ such that } x' \prec x \,\}$
\end{definition}
The \emph{Pareto front} is the image of the Pareto set in objective space:
$\mathcal{F}^\star = \{\, \mathbf{F}(x) \mid x \in \mathcal{X}^\star \,\}$.
Our goal is to generate a sequence of observations $\mathcal{D}_t$ such that the Pareto set approximation is close to the true Pareto set $\mathcal{X}^\star$, ensuring dense coverage of the optimal trade-off surface. To quantify solution quality, the Hypervolume (HV) indicator is typically used.

\begin{definition}[\textit{Hypervolume Indicator}]
Let $\mathcal{Y} \subset \mathbb{R}^M$ be a set of objective vectors and $\mathbf{r} \in \mathbb{R}^M$ be a reference point dominated by all $\mathbf{y} \in \mathcal{Y}$. The Hypervolume indicator $HV(\mathcal{Y})$ is defined as the Lebesgue measure $\lambda_{HV}$ of the union of hypercubes formed by each point $\mathbf{y}$ and the reference point $\mathbf{r}$:
\[
HV(\mathcal{Y}) = \lambda_{HV} \left( \bigcup_{\mathbf{y} \in \mathcal{Y}} \{ \mathbf{z} \in \mathbb{R}^M \mid \mathbf{y} \preceq \mathbf{z} \preceq \mathbf{r} \} \right).
\]
\end{definition}

\subsection{Related Work}
\label{sec:related_work}

\paragraph{Evolutionary Algorithms (EAs).} EAs are a dominant class of population-based heuristics for multi-objective optimization, typically categorized by their selection mechanism. Dominance-based methods \cite{deb2002fast,deb2014evolutionary}, rely on non-dominated sorting to guide the population toward the Pareto front. MOEA/D \cite{zhang2007moead} decomposes the problem into multiple single objective sub-problems. The aforementioned methods struggle with more than two objectives. To mitigate this, indicator-based methods (e.g., SMS-EMOA \cite{beume2007sms}) directly optimize a quality metric such as Hypervolume, offering robust performance at the cost of higher computational complexity.

\paragraph{Bayesian Optimization (BO).}
BO is a sample-efficient framework for optimizing expensive black-box functions by iteratively updating a probabilistic surrogate and maximizing an acquisition function \cite{shahriari2016taking}. In the multi-objective setting (MOBO), methods typically aim to approximate the Pareto front by maximizing indicators such as Hypervolume. While early approaches relied on scalarization \cite{knowles2006parego}, Expected Hypervolume Improvement (EHVI) \cite{emmerich11} has become the de-facto choice by directly targeting hypervolume maximization. Modern implementations like qEHVI \cite{daulton2020differentiable,daulton2021parallel} further enhance efficiency through parallelizable Monte Carlo integration.

\paragraph{LLM-based Global Optimization.}
Recent research leverages LLMs to guide global optimization by framing the search process as an iterative language generation task \cite{liu2024large,yang2024large,song24}. In particular, LLAMBO \cite{liu2024large} show that LLMs can act as effective candidate generator and surrogate model for sample-efficient single-objective optimization. However, purely LLM-based approaches often struggle with numerical reasoning in continuous domains, a limitation that \cite{schwanke2025improving} tries to address through space partitioning.


\paragraph{Search Space Partitioning.}
Search space partitioning has been shown to be beneficial especially for high-dimensional multi-objective optimization problems. Methods such as MORBO \cite{daulton22} use trust regions, while LaMOO~\cite{zhao2022multiobjective} uses adaptive tree-based partitioning based on sample quality. All such instances can be formulated using the Multi-Armed Bandit (MAB) framework, which provides a foundation for region selection, traditionally using axis-aligned partitions~\cite{munos11,bubeck11}. Our work extends these principles to LLM-based optimization, treating each partition as an arm in the MAB problem~\cite{valko13}.

\input{method}

\input{theory}

\section{Numerical Experiments}\label{sec:experiments}

\paragraph{Benchmarks.} We evaluate MOHOLLM on 15 low- to mid-dimensional multi-objective (up to 4 objectives) problems from the \texttt{pymoo}~\cite{pymoo} and \texttt{BoTorch}~\cite{botorch} libraries: 12 synthetic (DTLZ1-3, BraninCurrin, ChankongHaimes, GMM, Poloni, SchafferN1-2, TestFunction4, ToyRobust, Kursawe) and 3 real-world (Penicillin fermentation, VehicleSafety, CarSideImpact). See Appendix~\ref{appendix:benchmark_details} for details on each benchmark.

\paragraph{Baselines.} 
We compare against 2 standard MOBO methods and 11 MOEAs. Full details on these methods are provided in Appendix~\ref{appendix:baselines}.
Additionally, to clearly encapsulate the effect of our algorithmic design choices, we also compare to a global LLM baseline (see Algorithm~\ref{alg:global_llm}), that uses the exact same prompt templates and evaluation scheme as MOHOLLM, but always has only a global view of the input space.

\paragraph{Configuration.} Unless stated otherwise, MOHOLLM uses \texttt{Gemini-2.0-Flash} with the following default hyperparameters: $k=5$ regions, $N=5$ candidates per region, $b=4$ evaluations per trial, $m_0=5$, $\alpha_{\max}=1.0$ (decaying to 0.01), $\beta_1=\beta_2=0.5$. 
In Appendix~\ref{secapp:experiments}, we study the robustness of MOHOLLM towards these hyperparameter choices. We utilize the prompts templates as illustrated in Section~\ref{appendix:prompt_design}. We run each method for 50 function evaluations (including the 5 initial random samples) 10 times with different random seeds and report the mean HV $\pm$ 95\% confidence interval (CI).

\begin{figure*}[ht]
\centering
\begin{subfigure}{0.19\linewidth}
    \centering
    \includegraphics[width=\linewidth]{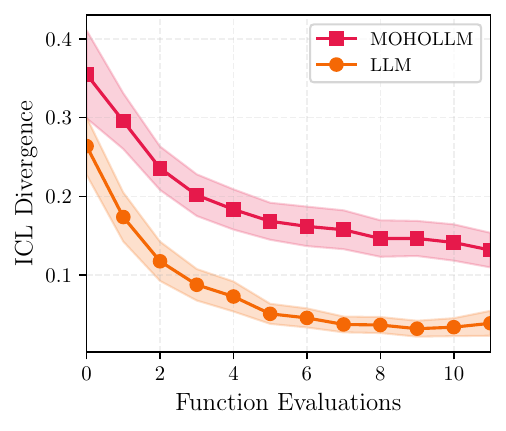}
    \caption{}
    \label{fig:icl_divergence_agg}
\end{subfigure}\hfill
\begin{subfigure}{0.19\linewidth}
    \centering
    \includegraphics[width=\linewidth]{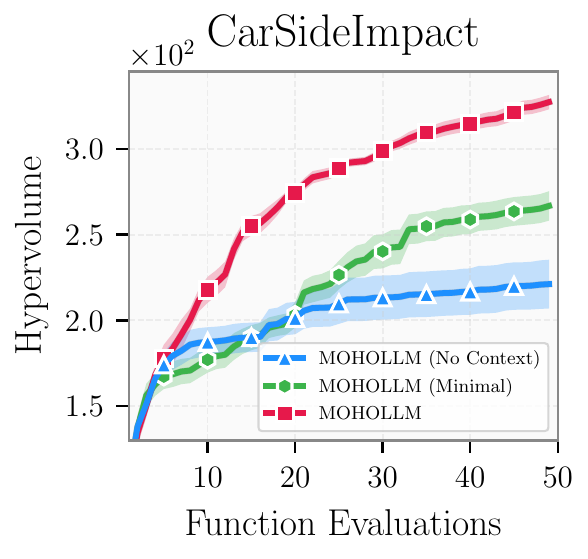}
    \caption{}
    \label{fig:prompt_ablations}
\end{subfigure}\hfill
\begin{subfigure}{0.19\linewidth}
    \centering
    \includegraphics[width=\linewidth]{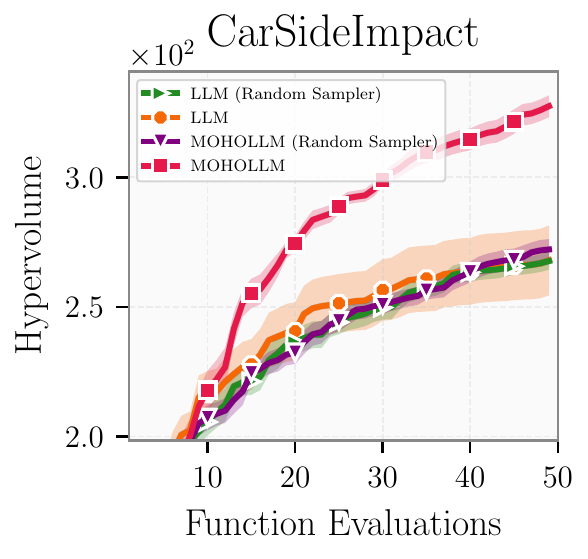}
    \caption{}
    \label{fig:candidate_sampler_ablation}
\end{subfigure}\hfill
\begin{subfigure}{0.19\linewidth}
    \centering
    \includegraphics[width=\linewidth]{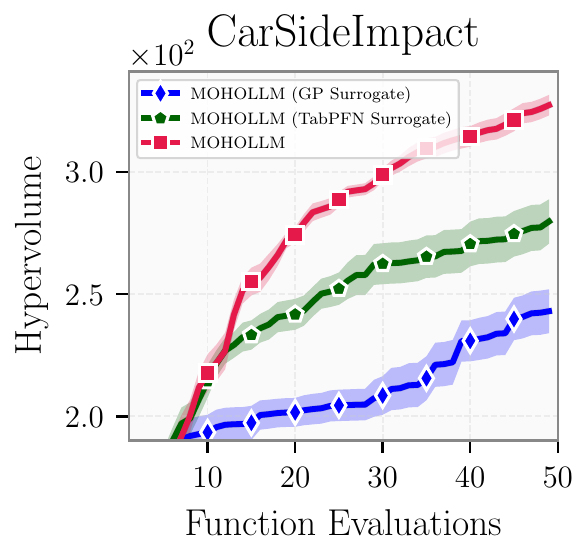}
    \caption{}
    \label{fig:surrogate_model_ablation}
\end{subfigure}\hfill
\begin{subfigure}{0.17\linewidth}
    \centering
    \includegraphics[width=\linewidth]{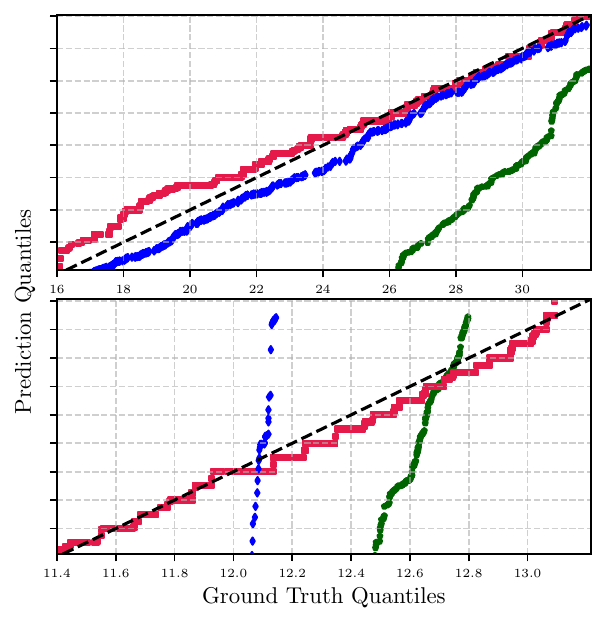}
    \caption{}
    \label{fig:qq_carside}
\end{subfigure}
\vspace{-2mm}
\caption{Ablation and analysis results.
(a) Aggregate ICL divergence over function evaluations.
(b) Prompt ablations.
(c) Candidate sampler ablation comparing LLM-based and random sampling strategies.
(d) Surrogate model ablation comparing LLM, GP, and TabPFN-v2 surrogates.
(e) Q--Q plot on 2 objectives from CarSideImpact, comparing predicted versus true objective value distributions.}
\label{fig:ablations_analysis}
\end{figure*}

\subsection{Results on Synthetic Benchmarks}

Results in Figure~\ref{fig:hv_synth} (see Figure~\ref{fig:hv_synth_all} in Appendix~\ref{secapp:experiments} for results on all benchmarks) show that MOHOLLM demonstrates consistent improvement compared to the global LLM baseline throughout the full evaluation budget across most benchmarks. 
This indicates that MOHOLLM is able find diverse solution sets across problems with different Pareto front characteristics -- from the linear and convex fronts of DTLZ1, Chankong–Haimes and SchafferN1 to the disconnected, highly nonconvex fronts of Branin-Currin and Kursawe. 
Although MOHOLLM does not surpass the competitive MOBO methods, which are known to work particularly well for small evaluation budgets, its performance places it within the competitive range of established MOEAs.

\subsection{Results on Real-World Benchmarks}

Real-world applications usually involve expensive-to-evaluate objectives, e.g., mechanical simulations or chemical processes. We consider the problems of optimizing VehicleSafety for side-impact collisions based on established crash simulations (CarSideImpact; $d=7$, $M=4$) and its crashworthiness design during a frontal impact (VehicleSafety, $d=5$, $M=3$)~\cite{liang2021scalable}. Conflicting objectives involve balancing the weight of a car (low weight brings lower manufacturing cost and less fuel consumption) and its damage to both the car and passengers during collision (lighter materials leads to more damage). We also evaluate MOHOLLM on the Penicillin production problem ($d=7$, $M=3$)~\cite{tanabe20}, where the objective is to maximize penicillin yield while simultaneously minimizing fermentation time and CO$_2$ emissions by optimizing seven initial reaction conditions.

As shown in Figure~\ref{fig:hv_real} (see Figure~\ref{fig:hv_real_all} for results with more MOEAs and Figure~\ref{fig:eaf} for the Pareto front plots), the global LLM baseline likely gets stuck exploiting sub-optimal regions, hence explaining low and stagnating HV on all tasks. On the other hand, although MOBO methods~\cite{daulton2020differentiable,emmerich11} remain superior on real-world tasks due to their sample efficiency, MOHOLLM consistently shows stronger performance than all MOEAs while avoiding the premature convergence of the global LLM-based search.






\subsection{Analysis and Ablations}\label{sec:analysis}

\paragraph{Impact of LLM type/size.}
Testing different LLMs on VehicleSafety (Figure~\ref{fig:hv_real_other}) shows a hierarchy in performance. Top performers (Gemini-2.0-Flash, GPT-OSS-120B, LLaMA-3.3-70B) are larger models optimized for instruction-following and reasoning. Smaller models (LLaMA-3.1-8B, Gemma-2-9B, Qwen3-32B) show significantly worse performance, suggesting that the LLM size and capabilities are relevant for such tasks. In Appendix~\ref{appendix:model_costs_api_usage} we provide operational costs' breakdown for using each model.

\paragraph{Sample Diversity.} 
To quantify exploration we compute the average distance of generated proposals from the historical evaluated examples: $D_{\text{ICL}} = 1/|S| \sum_{s \in S} \min_{i \in I} d(s, i)$, where $S$ are new proposals, $I$ are the examples, and $d(\cdot,\cdot)$ is Euclidean distance in the normalized unit hypercube. We call this metric in-context learning (ICL) divergence. Figure~\ref{fig:icl_divergence_agg} plots the aggregated ICL divergence across all benchmarks. We can see that MOHOLLM generates more diverse (higher ICL div.) candidates at each iteration compared to the global LLM sampler, demonstrating that the search space partitioning enforces more in-region exploration and prevents mode collapse (also in line with Assumption~\ref{ass:llm} from Section~\ref{sec:theory}).

\paragraph{Prompt Variations.}
We evaluate \texttt{Context} (default MOHOLLM setting with problem description in prompts), \texttt{No Context}, and \texttt{Minimal} prompts to assess the utility of domain knowledge in optimization (prompt variations shown in Appendix~\ref{appendix:prompt_design}). Real-world benchmark results in Figure~\ref{fig:prompt_ablations} and Figure~\ref{fig:hv_real_prompt_ablation_all} indicate that providing task context in the prompts typically results in better performance of MOHOLLM.

\paragraph{Impact of LLMs as Surrogate Models and Candidate Samplers.}
We also quantify the LLM contribution by isolating its roles in MOHOLLM. \textit{As a candidate sampler}, we see from Figure~\ref{fig:candidate_sampler_ablation} that MOHOLLM significantly outperforms its variant with random uniform sampling, both globally or inside input space partitions, confirming its ability to generate more informative proposals. \textit{As a surrogate model}, the LLM (Gemini-2.0-Flash) demonstrates reliable predictive performance on the real-world tasks, as shown in Figure~\ref{fig:surrogate_model_ablation} and the Q-Q plot in Figure~\ref{fig:qq_carside}, achieving a high Spearman rank correlation ($0.874$) and $R^2$ ($0.757$). In contrast, Gaussian Processes~\cite{Rasmussen2006Gaussian} and TabPFN-v2~\cite{hollmann2025tabpfn} fail to capture the objective structure, exhibiting near-zero correlations and highly negative $R^2$ scores in some objectives. We provide additional results on the other real-world benchmarks in Appendix~\ref{secapp:experiments}.

\section{Conclusion and Discussion}

We proposed MOHOLLM, a method that improves LLM-based multi-objective optimization through hierarchical space partitioning and a multi-armed bandit-insipired scoring for the resulting regions. Extensive experimental evaluations across 15 benchmarks, demonstrated that MOHOLLM showed consistent improvements compared to the global LLM-based baseline, while at the same time being competitive with established multi-objective evolutionary algorithms. We expect that MOHOLLM will open a pathway towards faithful integration of LLMs in multi-objective optimization.

Our method has also some limitations. The current reliance on axis-aligned KD-trees makes applicability to non-Euclidean domains not trivial. Furthermore, LLM inference cost increase with the problem dimensionality, making it hard for MOHOLLM to run on high evaluation budgets. Finally, while we established Pareto consistency, deriving formal asymptotic Hypervolume regret bounds~\cite{zhao2022multiobjective,golovin20} remains an open future problem.

\bibliographystyle{named}
\bibliography{ijcai26,hollm}

\clearpage
\newpage

\section*{Acknowledgments}
Frank Hutter, Arber Zela and David Salinas acknowledge the financial support of the Hector Foundation.

\appendix


\section{Details on Multi-objective Tasks}
\label{appendix:benchmark_details}

We evaluate MOHOLLM on a diverse suite of 15 benchmark problems designed to cover a wide range of structural properties encountered in multi-objective optimization.

\subsection{Synthetic Tasks}

We utilize the implementation from \texttt{pymoo}~\footnote{\url{https://pymoo.org/problems/test_problems.html}} for the synthetic tasks.

\paragraph{DTLZ1–3 ($d=6,\ M=2$).}
The DTLZ test suite defines scalable $M$-objective problems over decision variables $\mathbf{x}=(x_1,\dots,x_d)\in[0,1]^d$.
For all DTLZ problems, let $k=d-M+1$ and
$$
g(\mathbf{x}_M)=\sum_{i=M}^{d}\left[(x_i-0.5)^2-\cos\!\left(20\pi(x_i-0.5)\right)\right].
$$
\emph{DTLZ1} is defined as
$$
f_m(\mathbf{x})=\frac{1}{2}(1+g)\prod_{i=1}^{M-m} x_i \times
\begin{cases}
1-x_{M-m+1}, & m>1,\\
1, & m=1,
\end{cases}
$$
yielding a linear Pareto front with many local optima induced by $g$.
\emph{DTLZ2} replaces $g$ with $g(\mathbf{x}_M)=\sum_{i=M}^{d}(x_i-0.5)^2$ and defines
\begin{small}
$$
f_m(\mathbf{x})=(1+g)\prod_{i=1}^{M-m}\cos\!\left(\tfrac{\pi}{2}x_i\right)
\times
\begin{cases}
\sin\!\left(\tfrac{\pi}{2}x_{M-m+1}\right), & m>1,\\
1, & m=1,
\end{cases}
$$    
\end{small}
resulting in a smooth spherical Pareto front.
\emph{DTLZ3} uses the same objectives as DTLZ2 but the multimodal $g$ of DTLZ1, introducing many deceptive local fronts.

\paragraph{Branin–Currin ($d=2,\ M=2$).}
This bi-objective problem is defined over $\mathbf{x}=(x_1,x_2)\in[0,1]^2$.
The first objective is a scaled Branin function:
\begin{small}
\begin{align*}
    f_1(\mathbf{x})= &
\left(x_2-\frac{5.1}{4\pi^2}x_1^2+\frac{5}{\pi}x_1-6\right)^2 \\
&+10\left(1-\frac{1}{8\pi}\right)\cos(x_1)+10,
\end{align*}
\end{small}
and the second is the Currin function:
\begin{small}
$$
f_2(\mathbf{x})=
\left(1-\exp\!\left(-\frac{1}{2x_2}\right)\right)
\frac{2300x_1^3+1900x_1^2+2092x_1+60}
{100x_1^3+500x_1^2+4x_1+20}
$$
\end{small}
The conflicting landscapes yield a highly nonconvex Pareto front.

\paragraph{Chankong–Haimes ($d=2,\ M=2$).}
A 2-dimensional, 2-objective optimization benchmark defined on the domain $[-20, 20]^2$. We use the unconstrained version of this problem, dropping the original constraints, therefore resulting in a convex Pareto front. The objective functions are given by:
\begin{align*}
    &f_1(x, y) = 2 + (x - 2)^2 + (y - 1)^2\\
    &f_2(x, y) = 9x - (y - 1)^2
\end{align*}

\paragraph{GMM ($d=2,\ M=2$).}
Defined on the input unit square $[0, 1]^2$, it features two objectives, each independently defined by the a Gaussian Mixture Model (GMM). The problem is designed for minimization and incorporates multiplicative input noise $\xi \sim \mathcal{N}(\mu=[1,1], \Sigma=0.07I_2)$, applied before each function evaluation.

\paragraph{Poloni ($d=2,\ M=2$).} Defined on the domain $[-\pi, \pi]^2$, it features a non-convex, disconnected Pareto front. The objective functions are given by:
\begin{align*}
    f_1(x, y) &= 1 + (A_1 - B_1(x, y))^2 + (A_2 - B_2(x, y))^2 \\
    f_2(x, y) &= (x + 3)^2 + (y + 1)^2 
\end{align*}
where the constants $A_1, A_2$ and terms $B_1, B_2$ are defined as:
\begin{align*}
    A_1 &= 0.5 \sin(1) - 2 \cos(1) + \sin(2) - 1.5 \cos(2) \\
    A_2 &= 1.5 \sin(1) - \cos(1) + 2 \sin(2) - 0.5 \cos(2) \\
    B_1(x, y) &= 0.5 \sin(x) - 2 \cos(x) + \sin(y) - 1.5 \cos(y) \\
    B_2(x, y) &= 1.5 \sin(x) - \cos(x) + 2 \sin(y) - 0.5 \cos(y)
\end{align*}

\paragraph{SchafferN1 ($d=1,\ M=2$).} A classic, 1-dimensional, 2-objective optimization benchmark defined on the domain $[-10, 10]$ (parameter $A=10$). It is known for its simple structure but challenging convergence properties. The objectives are:
\begin{align*}
    f_1(x) &= x^2 \\
    f_2(x) &= (x - 2)^2 
\end{align*}

\paragraph{SchafferN2 ($d=1,\ M=2$).} Defined on the domain $[-5, 10]$, SchafferN2 features a piecewise linear first objective, adding complexity to the optimization landscape. The objectives are:
\begin{align*}
    f_1(x) &=
    \begin{cases}
        -x & \text{if } x \le 1 \\
        x - 2 & \text{if } 1 < x \le 3 \\
        4 - x & \text{if } 3 < x \le 4 \\
        x - 4 & \text{if } x > 4
    \end{cases} \\
    f_2(x) &= (x - 5)^2 
\end{align*}

\paragraph{TestFunction4 ($d=2,\ M=2$).} Defined on the domain $[-7, 4]^2$. While the original problem includes linear constraints, we utilize the unconstrained version in this work. The objective functions are:
\begin{align*}
    f_1(x, y) &= x^2 - y \\
    f_2(x, y) &= -0.5x - y - 1
\end{align*}

\paragraph{ToyRobust ($d=1,\ M=2$).} A benchmark defined on the domain $[0.0, 0.7]$, designed such that the Pareto front is sensitive to input noise and is formulated for minimization. The first objective function $f_1(x)$ is defined as:
\begin{align*}
    f_1(x) = &-30 \left( p_1(x) \cdot \sigma(x) + p_2(x) \cdot (1 - \sigma(x)) + s(x) + 1 \right) \\
    \text{where } &\quad p_1(x) = \ 2.4 - 10x - 0.1x^2 \\
    &\quad p_2(x) = \ 2x - 0.1x^2 \\
    &\quad s(x) = \ (x - 0.5)^2 + 0.1 \sin(30x) \\
    &\quad \sigma(x) = \ \text{sigmoid}((0.2 - x) / 0.005) 
\end{align*}
The second objective function $f_2(x)$ uses a modified 2-dimensional Levy function, applied to a transformed input $x' = 0.95x + 0.03$ with the second dimension fixed at 0:
\begin{align*}
    f_2(x) &= \text{Levy}(x', 0) - 0.75 (x')^2 
\end{align*}
The $n$-dimensional Levy function $f_{\text{Levy}}(z_1, \dots, z_d)$ is generally defined as:
\begin{align*}
    f_{\text{Levy}}(z) &= \sin^2(\pi w_1) \\
        &\quad + \sum_{i=1}^{n-1} (w_i-1)^2 (1 + 10 \sin^2(\pi w_i + 1)) \\
        &\quad + (w_n - 1)^2 (1 + \sin^2(2 \pi w_n)) \\
    \text{where } w_i &= 1 + \frac{z_i - 1}{4} \quad \text{for } i = 1, \dots, d 
\end{align*}
Here, $n=2$, $z_1 = x'$, and $z_2 = 0$.

\paragraph{Kursawe ($d=3,\ M=2$).} A 3-dimensional, 2-objective optimization problem defined on the domain $[-5, 5]^3$. It is known for its disconnected Pareto front, making it deceptive and challenging for optimizers regarding solution diversity. The 2 objectives are:
\begin{align*}
    f_1(x_1, x_2, x_3) &= \sum_{i=1}^{2} -10 \exp\left(-0.2 \sqrt{x_i^2 + x_{i+1}^2}\right) \\
    f_2(x_1, x_2, x_3) &= \sum_{i=1}^{3} \left( |x_i|^{0.8} + 5 \sin(x_i^3) \right)
\end{align*}

\subsection{Real-world Benchmarks}
We utilize the benchmark implementations provided by the \texttt{BoTorch} library\footnote{\url{https://github.com/meta-pytorch/botorch/blob/main/botorch/test_functions/multi_objective.py}}. These problems represent challenging black-box optimization scenarios characterized by high-dimensional continuous search spaces and conflicting objectives.

\paragraph{Penicillin Fermentation ($d=7, M=3$)}\cite{liang2021scalable}.
This benchmark simulates a fed-batch penicillin production process governed by a system of nonlinear ordinary differential equations (ODEs). The decision vector $\mathbf{x} \in \mathbb{R}^7$ controls the initial culture volume ($V$), biomass concentration ($X$), temperature ($T$), glucose concentration ($S$), substrate feed rate ($F$), substrate feed concentration ($s_f$), and pH ($H^+$).
The simulation evolves over discrete time steps $t$, updating state variables according to the following dynamics (simplified for brevity):
\begin{align*}
    &\frac{dV}{dt} = F - V \lambda (e^{5 \frac{T-T_o}{T_v-T_o}} - 1) \\
    &\frac{dX}{dt} = \mu(\mathbf{x}) X - \frac{X}{V}\frac{dV}{dt} \\
    &\frac{dS}{dt} = -\left(\frac{\mu}{Y_{xs}} + \frac{\mu_p}{Y_{ps}} + m_X\right)X + \frac{F s_f}{V} - \frac{S}{V}\frac{dV}{dt} \\
    &\frac{dP}{dt} = \mu_p X - K P - \frac{P}{V}\frac{dV}{dt} \\
    &\frac{dCO_2}{dt} = \alpha_1 \frac{dX}{dt} + \alpha_2 X + \alpha_3
\end{align*}
where $\mu$ and $\mu_p$ are nonlinear growth rates dependent on $T, S,$ and $H$. The optimization minimizes three objectives based on the final state at time $T_{final}$:
\begin{align*}
    f_1(\mathbf{x}) &= -P(T_{final}) \\
    f_2(\mathbf{x}) &= CO_2(T_{final}) \\
    f_3(\mathbf{x}) &= T_{final}
\end{align*}
representing the negative penicillin yield (maximization), total CO$_2$ production, and total fermentation time.

\paragraph{VehicleSafety ($d=5, M=3$)}\cite{tanabe20}.
This problem optimizes vehicle crashworthiness during a frontal impact. The input $\mathbf{x} \in [1, 3]^5$ represents the thickness of five structural components: bumper beam ($x_1$), crash box ($x_2$), longitudinal rails ($x_3$), A-pillar ($x_4$), and dash panel ($x_5$).

The three minimization objectives are defined as:
\begin{align*}
    f_1(\mathbf{x}) &= 1640.2823 + 2.3573285x_1 + 2.3220035x_2 \\
    &\quad + 4.5688768x_3 + 7.7213633x_4 + 4.4559504x_5 \\
    f_2(\mathbf{x}) &= 6.5856 + 1.15x_1 - 1.0427x_2 + 0.9738x_3 \\
    &\quad + 0.8364x_4 - 0.3695x_1x_4 + 0.0861x_1x_5 \\
    &\quad + 0.3628x_2x_4 - 0.1106x_1^2 - 0.3437x_3^2 \\
    &\quad + 0.1764x_4^2 \\
    f_3(\mathbf{x}) &= -0.0551 + 0.0181x_1 + 0.1024x_2 + 0.0421x_3 \\
    &\quad - 0.0073x_1x_2 + 0.024x_2x_3 - 0.0118x_2x_4 \\
    &\quad - 0.0204x_3x_4 - 0.008x_3x_5 - 0.0241x_2^2 \\
    &\quad + 0.0109x_4^2
\end{align*}
corresponding to Mass ($f_1$), Acceleration ($f_2$), and Intrusion ($f_3$).

\paragraph{CarSideImpact ($d=7, M=4$)}\cite{tanabe20}.
This benchmark addresses side-impact safety. Inputs $\mathbf{x} \in \mathbb{R}^7$ define the thickness of components like B-pillars and door beams.

The problem defines four minimization objectives:
\begin{align*}
    f_1(\mathbf{x}) &= 1.98 + 4.9x_1 + 6.67x_2 + 6.98x_3 + 4.01x_4 \\
    &\quad + 1.78x_5 + 10^{-5}x_6 + 2.73x_7 \\
    f_2(\mathbf{x}) &= 4.72 - 0.5x_4 - 0.19x_2x_3 \\
    f_3(\mathbf{x}) &= 0.5 \cdot (V_{MBP} + V_{FD}) \\
    f_4(\mathbf{x}) &= \sum_{i=1}^{10} \max(0, -g_i(\mathbf{x}))
\end{align*}
where $f_1$ is Weight, $f_2$ is Pubic Force, and $f_3$ is the average intrusion velocity derived from $V_{MBP}$ and $V_{FD}$. $f_4$ aggregates violations of 10 structural constraints $g_i(\mathbf{x})$.

\subsection{Hypervolume Reference Points}\label{appendix:hypervolume_reference_points}

The Hypervolume is computed using the reference points shown in Table~\ref{tab:hv_reference_points} for each benchmark. These reference points are chosen to be dominated by all points on the Pareto front, ensuring meaningful Hypervolume calculations. The reference points on the real-world benchmark (Penicillin, VehicleSafety and CarSideImpact) have been taken from the \texttt{pymoo} \cite{pymoo} library. The reference points of the other benchmarks have been computed by identifying the worst point in the pool of all proposals across all runs and methods with an $1\%$ offset.

\begin{table}[ht]
\centering
\caption{Hypervolume reference points used for performance evaluation across all benchmarks.}
\label{tab:hv_reference_points}
\resizebox{\linewidth}{!}{%
\begin{tabular}{@{}lc@{}}
\toprule
\textbf{Benchmark} & \textbf{Reference Point} \\
\midrule
DTLZ1 & $[510.57419, 528.80469]$ \\
DTLZ2 & $[2.2725, 2.2725]$ \\
DTLZ3 & $[1109.84052, 1109.84052]$ \\
BraninCurrin & $[311.21029, 13.91174]$ \\
ChankongHaimes & $[936.27, 180.3557]$ \\
GMM & $[0.0, 0.0]$ \\
Poloni & $[62.2463, 52.57454]$ \\
SchafferN1 & $[101.0, 145.44]$ \\
SchafferN2 & $[6.06, 101.0]$ \\
TestFunction4 & $[56.56, 9.595]$ \\
ToyRobust & $[49.995, 37.36394]$ \\
Kursawe & $[-4.91062, 24.01174]$ \\
\midrule
Penicillin & $[25.935, 57.612, 935.5]$ \\
CarSideImpact & $[45.4872, 4.5114, 13.3394, 10.3942]$ \\
VehicleSafety & $[1864.72022, 11.81993945, 0.2903999384]$ \\
\bottomrule
\end{tabular}
}
\end{table}

For the real-world benchmarks, the maximum achievable hypervolume is documented in Table~\ref{tab:max_hypervolume}. The maximum hypervolume was taken from the \texttt{BoTorch} framework.

\begin{table}[ht]
\centering
\caption{Maximum achievable hypervolume for real-world benchmarks according to the \texttt{BoTorch} framework.}
\label{tab:max_hypervolume}
\begin{tabular}{@{}lc@{}}
\toprule
\textbf{Benchmark} & \textbf{Maximum Hypervolume} \\
\midrule
Penicillin & $2{,}183{,}455.909507436$ \\
CarSideImpact & $484.72654347642793$ \\
VehicleSafety & $246.81607081187002$ \\
\bottomrule
\end{tabular}
\end{table}

\section{Baselines}\label{appendix:baselines}
This section provides a brief description of the baseline algorithms used to compare MOHOLLM throughout this paper. The baselines are grouped into two categories: Multi-Objective Bayesian Optimization (MOBO) methods, primarily implemented in \texttt{BoTorch} \cite{botorch} and Multi-Objective Evolutionary Algorithms (MOEAs), implemented in \texttt{pymoo} \cite{pymoo} and \texttt{platypus}\footnote{\url{https://github.com/Project-Platypus/Platypus}}. We ran all library-based methods with their standard default parameters.

\begin{algorithm}[t]
\caption{Global LLM Baseline Pseudocode}
\label{alg:global_llm}
\begin{algorithmic}[1]
\REQUIRE Initial $\mathcal{D}_{t_0}$, budget $T$, batch size $b$, candidates $N$.
\STATE Initialize $t \leftarrow t_0$.
\WHILE{$t < T$}
    \STATE \textbf{// Step 1: Global Sampling}
    \STATE Generate $N$ candidates $\hat{x} \sim p_{\text{LLM}}(x \mid \mathcal{C}_{global}(\mathcal{D}_t))$
    
    \STATE \textbf{// Step 2: Evaluation}
    \STATE For all candidates predict $\hat{\mathbf{y}} \sim p_{\text{LLM}}(\mathbf{y} \mid \mathcal{C}(\mathcal{D}_t,\hat{x}))$
    \STATE Select batch $B^*$ of size $b$ maximizing predicted HV
    \STATE Evaluate true objectives: $\mathbf{Y}^{*} \leftarrow \mathbf{F}(B^*)$
    \STATE Update $\mathcal{D}_{t} \leftarrow \mathcal{D}_t \cup (B^*, \mathbf{Y}^{*})$
    \STATE $t \leftarrow t + b$
\ENDWHILE
\RETURN $\mathcal{F}_{final}$ (Pareto Front of $\mathcal{D}_T$)
\end{algorithmic}
\end{algorithm}

\paragraph{Multi-objective Bayesian Optimization Baselines.}
We ran two strong standard MOBO baselines.
\begin{description}
    \item[\texttt{EHVI}~{\normalfont\cite{emmerich11}}] The Expected Hypervolume Improvement is a classic MOBO acquisition function that selects the single point that is expected to provide the largest increase to the hypervolume of the current Pareto front approximation.
    \item[\texttt{qLogEHVI}~{\normalfont\cite{daulton2020differentiable,ament2023unexpected}}] A parallelized version of EHVI designed for improved numerical stability during acquisition function optimization. The $q$ indicates batch evaluation of $q$ points. The $Log$ refers to internal log-transformations and smooth approximations applied during the calculation, preventing the acquisition values and gradients from vanishing numerically, which may be problematic in the standard \texttt{qEHVI}~\cite{ament2023unexpected,daulton2020differentiable}.
\end{description}

\paragraph{Multi-objective Evolutionary Algorithm Baselines.}
We compare MOHOLLM against a comprehensive suite of 12 established MOEAs that are a standard choice in multi-objective optimization.
\begin{description}
    \item[\texttt{NSGA-II}~{\normalfont\cite{deb2002fast}}] The Non-dominated Sorting Genetic Algorithm II is one of the most widely used and influential MOEAs. It uses a non-dominated sorting procedure for exploitation and a crowding distance mechanism to promote diversity.

    \item[\texttt{NSGA-III}~{\normalfont\cite{deb2014evolutionary}}] An extension of \texttt{NSGA-II} designed to handle many-objective optimization problems. It replaces the crowding distance with a set of pre-defined reference directions to maintain diversity.

    \item[\texttt{SPEA2}~{\normalfont\cite{zitzler2001spea2}}] The Strength Pareto Evolutionary Algorithm 2 is a classic MOEA that uses a concept of strength, which is the number of solutions an individual dominates, and a density estimation technique to guide the search.

    \item[\texttt{MOEA/D}~{\normalfont\cite{zhang2007moead}}] The Multi-Objective Evolutionary Algorithm based on Decomposition is a method that decomposes the multi-objective problem into $N$ single-objective optimization subproblems, each defined by a different weight vector applied to an aggregation function (like Tchebycheff or weighted sum). The algorithm then optimizes these $N$ subproblems simultaneously in a collaborative manner by defining a \textit{neighborhood} for each subproblem based on the $T$ closest weight vectors. During optimization, parents for a subproblem are selected only from its neighborhood, and a new solution can replace the current solutions of its neighboring subproblems, allowing good solutions to spread locally.

    \item[\texttt{GDE3}~{\normalfont\cite{kukkonen2005gde3}}] The third version of Generalized Differential Evolution, which adapts the Differential Evolution (DE) algorithm for multi-objective constrained optimization. It improves upon earlier GDE versions by incorporating non-dominated sorting and crowding distance.

    \item[\texttt{IBEA}~{\normalfont\cite{zitzler2004indicator}}] The Indicator-Based Evolutionary Algorithm. It uses a quality indicator, such as hypervolume contribution, directly as the fitness function for selection.

    \item[\texttt{SMSEMOA}~{\normalfont\cite{beume2007sms}}] The S-Metric Selection Evolutionary Multi-objective Algorithm is a steady-state algorithm, meaning only one individual is replaced in the population each generation. It uses the hypervolume indicator (S-metric) contribution to select which solution to discard. Typically, the algorithm removes the individual contributing the least hypervolume from the worst non-dominated front to make room for a new offspring.
    
    \item[\texttt{PESA2}~{\normalfont\cite{corne2001pesa2}}] The Pareto Envelope-based Selection Algorithm 2 introduces a region-based selection mechanism for MOEAs. Instead of assigning fitness to individual solutions, PESA2 partitions the \textit{objective space} into hyperrectangles (regions) and assigns fitness to these regions based on the number of non-dominated solutions they contain in an external archive. A parent to mutate is then selected by first choosing a hyperrectangle that favors less crowded regions to promote diversity and then randomly selecting an individual from within that hyperrectangle.
    
    \item[\texttt{RNSGA-II}~{\normalfont\cite{deb2006reference}}] A variant of \texttt{NSGA-II} that incorporates reference points to guide the search and can be effective particularly for many-objective problems.

    \item[\texttt{UNSGA3}~{\normalfont\cite{seada2016unified}}] The Unified Non-dominated Sorting Genetic Algorithm III can efficiently tackle single-, multi-, and many-objective optimization problems. It modifies the \texttt{NSGA-III} approach primarily through a niching-based tournament selection mechanism.

    \item[\texttt{CTAEA}~{\normalfont\cite{li2018two}}] A \textit{Two-Archive} algorithm that maintains one archive for convergence and another for diversity, aiming to balance the two explicitly.
\end{description}

\section{Operational Costs and API Usage}\label{appendix:model_costs_api_usage}
In this section, we provide the API costs and providers for the LLMs used throughout the experiments in this paper. All costs are measured per million tokens and are based on the pricing as of the time of the experiments. The models were accessed through the following main providers: Google API~\footnote{\url{https://cloud.google.com/apis}}, Nebius AI~\footnote{\url{https://github.com/nebius/api}} and OpenAI API~\footnote{\url{https://platform.openai.com/docs/api-reference}}.

\begin{table}[ht]
    \centering
    \caption{API cost structure for language models used in the experiments.}
    \label{tab:model_costs}
    \resizebox{\linewidth}{!}{%
    \begin{tabular}{@{}llcc@{}}
    \toprule
        \textbf{Model} & \textbf{Provider} & \textbf{Input Cost} & \textbf{Output Cost} \\
        & & \textbf{(per 1M tokens)} & \textbf{(per 1M tokens)} \\
    \midrule
    Gemini-2.0-Flash & Google & \$0.10 & \$0.40 \\
    
    Gemma-2-2B & Nebius & \$0.10 & \$0.10 \\
    
    Gemma-2-9B & Nebius & \$0.10 & \$0.10 \\
    
    GPT-4o-mini & OpenAI & \$0.15 & \$0.60 \\
    
    GPT-OSS-120B & Nebius & \$0.60 & \$0.60 \\
    
    LLaMA-3.1-8B & Nebius & \$0.10 & \$0.10 \\
    
    LLaMA-3.3-70B & Nebius & \$0.60 & \$0.60 \\
    
    Qwen3-32B & Nebius & \$0.60 & \$0.60 \\
    \bottomrule
    \end{tabular}
    }
\end{table}

In Table~\ref{tab:model_operational_costs} we show the measured operational costs of MOHOLLM with different LLM models on the VehicleSafety benchmark. Gemini-2.0-Flash is the fastest model both per trial and overall, while LLaMA-3.1-8B achieves by far the lowest monetary cost. LLaMA-3.3-70B uses the fewest API requests and the lowest total tokens on average, but at higher runtime, whereas Qwen3-32B and GPT-OSS-120B have higher token usage, runtime, and cost.

\begin{table*}[ht]
    \centering
    \caption{A comparison of operational costs for various language models. Results show the mean ± standard error of 10 independent seeds on the VehicleSafety benchmark. The best performance (lowest value) for each metric is marked in \textbf{bold}.}
    \label{tab:model_operational_costs}
    \resizebox{\textwidth}{!}{%
    \begin{tabular}{@{}llllllll@{}}
        \toprule
         Model & \textbf{Gemini-2.0-Flash} & \textbf{LLaMA-3.3-70B} & \textbf{LLaMA-3.1-8B} & \textbf{Qwen3-32B} & \textbf{Gemma-2-9B} & \textbf{GPT-4o-mini} & \textbf{GPT-OSS-120B} \\
        \midrule
        \multicolumn{8}{l}{\textit{Average Per-Trial}} \\
        Time (s) & \textbf{20.59 ± 1.64} & 70.01 ± 5.42 & 117.24 ± 24.74 & 224.12 ± 30.08 & 37.04 ± 6.37 & 58.53 ± 5.03 & 250.28 ± 51.33 \\
        Prompt Tokens & \textbf{44387 ± 30286} & 56792 ± 8656 & 114776 ± 7250 & 67263 ± 7898 & 91503 ± 15878 & 58357 ± 4037 & 92596 ± 25016 \\
        Completion Tokens & \textbf{2923 ± 1872} & 3435 ± 103 & 7651 ± 1256 & 70826 ± 8140 & 4497 ± 514 & 3581 ± 111 & 63615 ± 11490 \\
        Total Tokens & \textbf{47310 ± 32139} & 60227 ± 8756 & 122427 ± 8008 & 138089 ± 15868 & 96000 ± 16387 & 61938 ± 4119 & 156212 ± 34091 \\
        API Requests & 21 ± 3.5 & \textbf{20 ± 2.2} & 39 ± 2.2 & 22 ± 2.7 & 28 ± 4.4 & 21 ± 1.2 & 32 ± 8.9 \\
        Cost ($\times 10^{-3} \$)$ & 7.918 ± 0.760 & 8.757 ± 1.165 & \textbf{2.755 ± 0.197} & 55.948 ± 6.413 & 3.150 ± 0.522 & 10.902 ± 0.656 & 52.059 ± 9.900 \\
        \midrule
        \multicolumn{8}{l}{\textit{Average Total}} \\
        Tokens & 933693 ± 223702 & \textbf{722723 ± 105074} & 1469130 ± 96102 & 1657063 ± 190416 & 1151996 ± 196649 & 743253 ± 49430 & 1874542 ± 409092 \\
        Cost (\$) & 0.10930 ± 0.02439 & 0.10508 ± 0.01398 & \textbf{0.03306 ± 0.00236} & 0.67138 ± 0.07696 & 0.03780 ± 0.00627 & 0.13082 ± 0.00787 & 0.62470 ± 0.11879 \\
        Time (s) & \textbf{300.10 ± 25.13} & 857.72 ± 79.88 & 1444.58 ± 330.08 & 2689.42 ± 360.97 & 444.51 ± 76.49 & 702.42 ± 60.34 & 3003.39 ± 615.91 \\
        \bottomrule
    \end{tabular}
    } 
\end{table*}

\section{Robustness and Instruction Adherence}
\label{sec:failure_analysis}

In this section, we analyze the instruction-following abilities of the LLM when generating a candidate by quantifying the duplicate (identical to another solution generated within the same batch), reobservation (identical to a point already evaluated in a previous iteration) and out-of-region rate of its candidate generation process. In the experiments, whenever the LLM samples a out-of-region point, we rerun the generative process until it samples a candidate within the region.

\begin{table*}[ht]
\centering
\caption{Analysis of model efficiency and rejection rates on the VehicleSafety benchmark (mean over 10 seeds). \textbf{LLaMA-3.1-8B} exhibits the highest failure rate, rejecting over half of all generated proposals.}
\label{tab:rejection_rate_analysis}
\resizebox{\textwidth}{!}{%
\begin{tabular}{@{}lccccccc@{}}
\toprule
\textbf{Model} 
& \textbf{Gemini-2.0-Flash} 
& \textbf{LLaMA-3.3-70B} 
& \textbf{LLaMA-3.1-8B} 
& \textbf{Qwen3-32B} 
& \textbf{Gemma-2-9B} 
& \textbf{GPT-4o-mini} 
& \textbf{GPT-OSS-120B} \\
\midrule
\textbf{Avg. Total Proposed} 
& \textbf{554.1} 
& 556.7 
& 1213.3 
& 599.0 
& 743.4 
& 577.8 
& 798.7 \\
\midrule
\multicolumn{8}{l}{\textit{Rejection Rate Breakdown (\%)}} \\
Duplicate 
& 0.08 
& \textbf{0.00} 
& 0.89 
& \textbf{0.00} 
& \textbf{0.00} 
& \textbf{0.00} 
& \textbf{0.00} \\
Re-observed 
& 1.01 
& 0.12 
& 1.49 
& \textbf{0.00} 
& 0.01 
& \textbf{0.00} 
& 1.49 \\
Out of Region 
& \textbf{4.85} 
& 6.09 
& 49.28 
& 11.69 
& 29.50 
& 8.08 
& 27.18 \\
\midrule
\textbf{Total Rejection Rate} 
& \textbf{5.93} 
& 6.21 
& 51.66 
& 11.69 
& 29.52 
& 8.08 
& 28.67 \\
\bottomrule
\end{tabular}
}
\end{table*}

In Table~\ref{tab:rejection_rate_analysis}, we analyze the rejection rates of MOHOLLM with different LLM models on the VehicleSafety benchmark. Results indicate that LLMs sometimes struggle to respect hard numerical boundaries, a difficulty that is amplified at later iterations when the partitions become increasingly narrow.
Within the LLaMA family, \textit{LLaMA-3.1-8B} fails on over 50\% of proposals, whereas the \textit{70B} variant shows only 6\% out-of-region rate. \textit{GPT-Oss 120B} suffers a rejection rate of nearly 30\%, significantly higher than the smaller LLaMA model. On the other hand, the smaller \textit{Gemma-2-9B} demonstrates superior constraint adherence compared to the 8B LLaMA model, highlighting that architectural choices and training data composition are as important as parameter count for numerical reasoning tasks.


\section{Additional Experimental Results}
\label{secapp:experiments}

In this section, we present additional experimental results illustrating the performance of MOHOLLM across both synthetic and real-world benchmark problems, more analysis results and ablations.


\begin{figure*}[ht]
    \centering
    \includegraphics[width=0.99\linewidth]{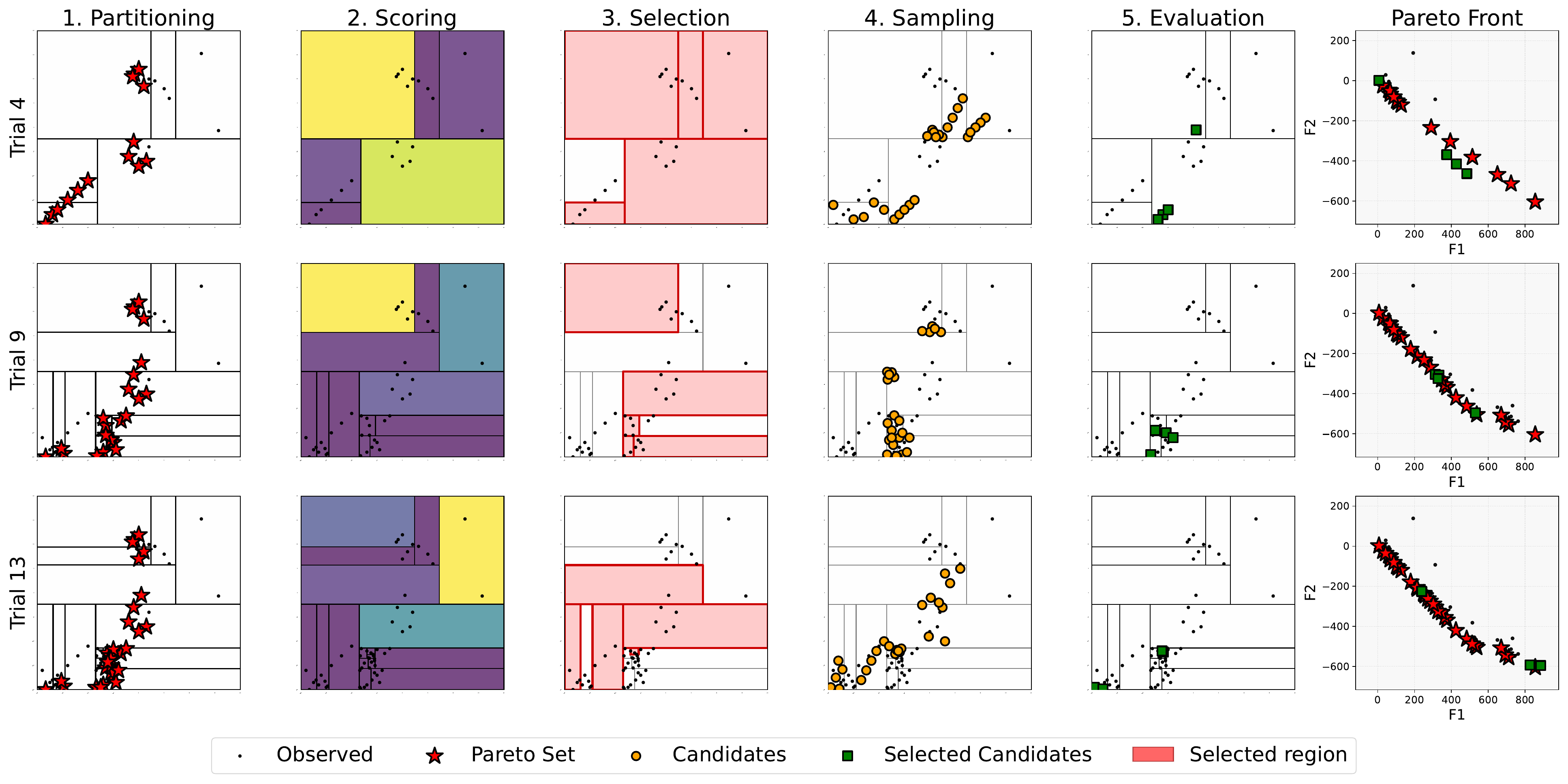}
    \caption{Search process of MOHOLLM on the 2D Chankong-Haimes function at different optimization stages. The last column illustrates the function values, while other columns the five algorithmic steps: partitioning, region scoring, probabilistic selection, LLM-based sampling, and evaluation. Rows correspond to early, intermediate, and late trials. MOHOLLM transitions from coarse, globally exploratory partitions to fine-grained, localized refinement around the Pareto front, allocating samples to both sparse regions and dense Pareto-optimal areas.}
    \label{fig:search_process_chankong}
\end{figure*}

\begin{figure*}[ht]
    \centering
    \includegraphics[width=0.99\linewidth]{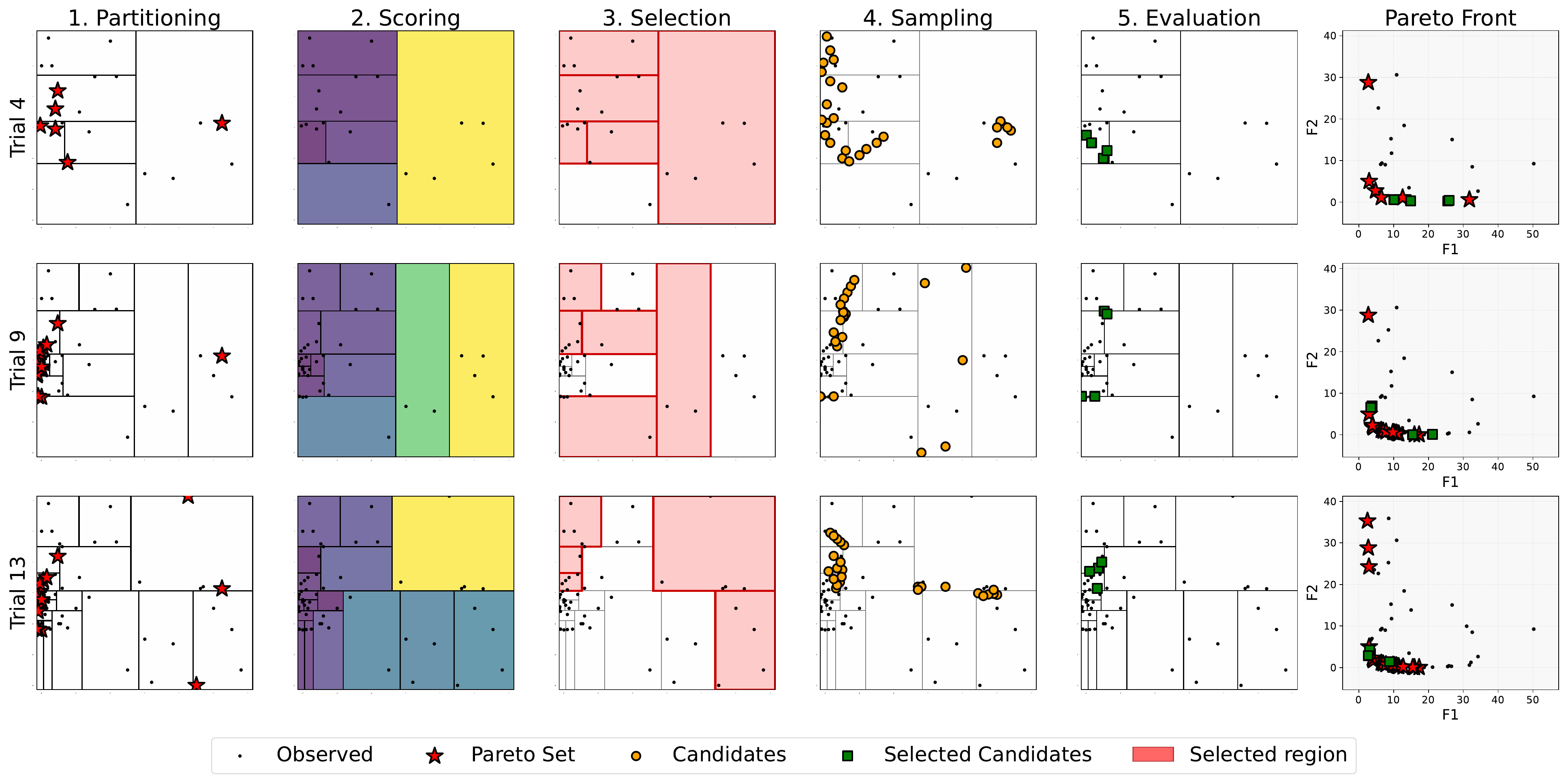}
    \caption{Search process of MOHOLLM on the 2D Poloni function at different optimization stages. The last column illustrates the function values, while other columns the five algorithmic steps: partitioning, region scoring, probabilistic selection, LLM-based sampling, and evaluation. Rows correspond to early, intermediate, and late trials. MOHOLLM transitions from coarse, globally exploratory partitions to fine-grained, localized refinement around the Pareto front, allocating samples to both sparse regions and dense Pareto-optimal areas.}
    \label{fig:search_process_poloni}
\end{figure*}


\begin{figure*}[ht]
    \centering
    \includegraphics[width=0.99\linewidth]{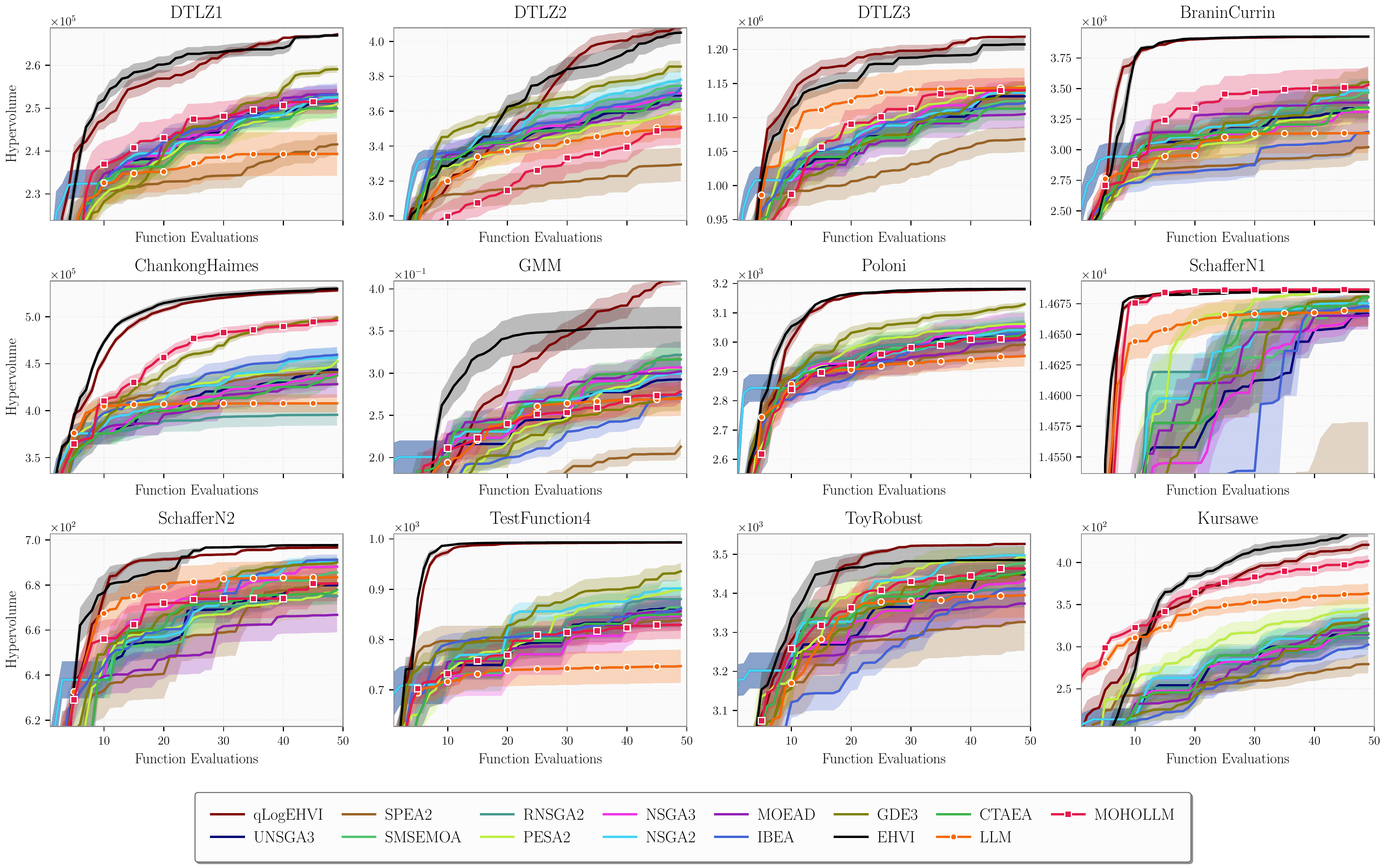}
    \caption{Hypervolume (HV) trajectory over function evaluations on the synthetic benchmark problems. Curves report the mean HV across 10 runs, and shaded regions indicating the $95\%$ CI. MOHOLLM shows competitive performance across all tasks, in contrast to the unconstrained LLM baseline, which rapidly stagnates.}
    \label{fig:hv_synth_all}
\end{figure*}

\begin{figure*}[ht]
    \centering
    \includegraphics[width=0.99\linewidth]{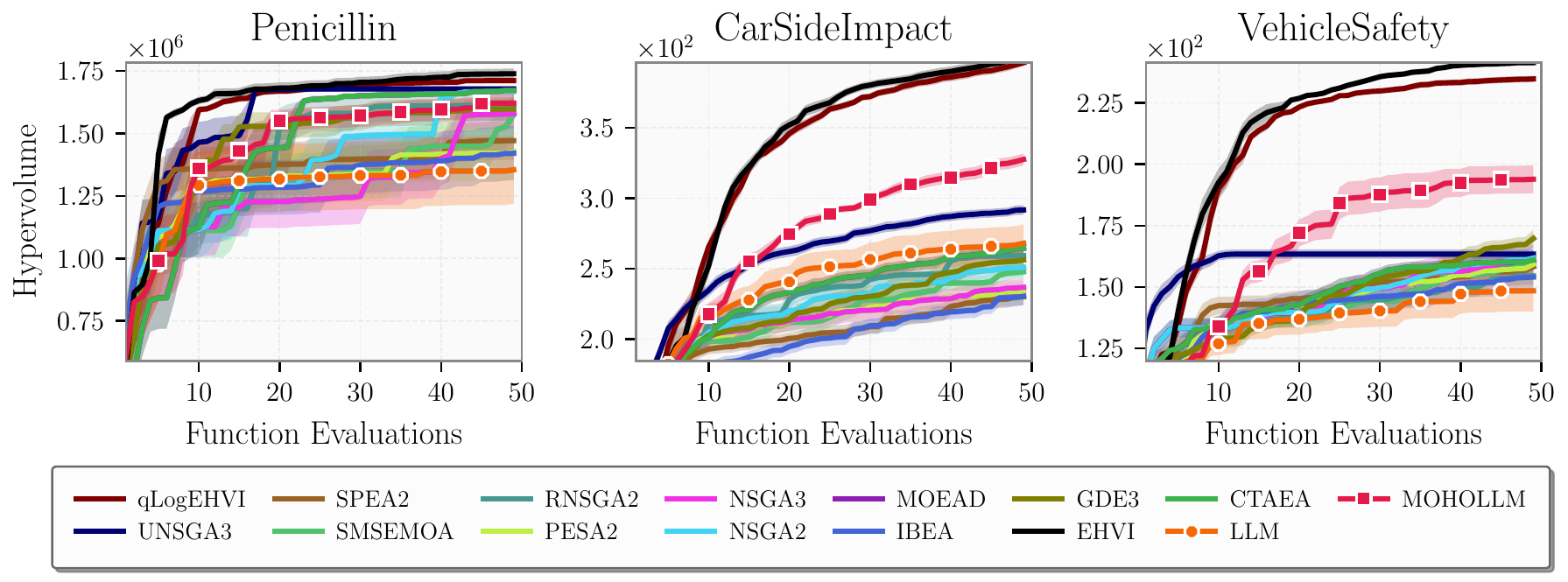}
    \caption{Hypervolume (HV) trajectory over function evaluations on real-world benchmarks (Penicillin, VehicleSafety, and CarSideImpact). Curves report the mean HV across 10 runs, and shaded regions indicating the $95\%$ CI. MOHOLLM shows competitive performance across all tasks, in contrast to the unconstrained LLM baseline, which rapidly stagnates..}
    \label{fig:hv_real_all}
\end{figure*}


\begin{figure*}[ht]
    \centering
    \includegraphics[width=0.93\linewidth]{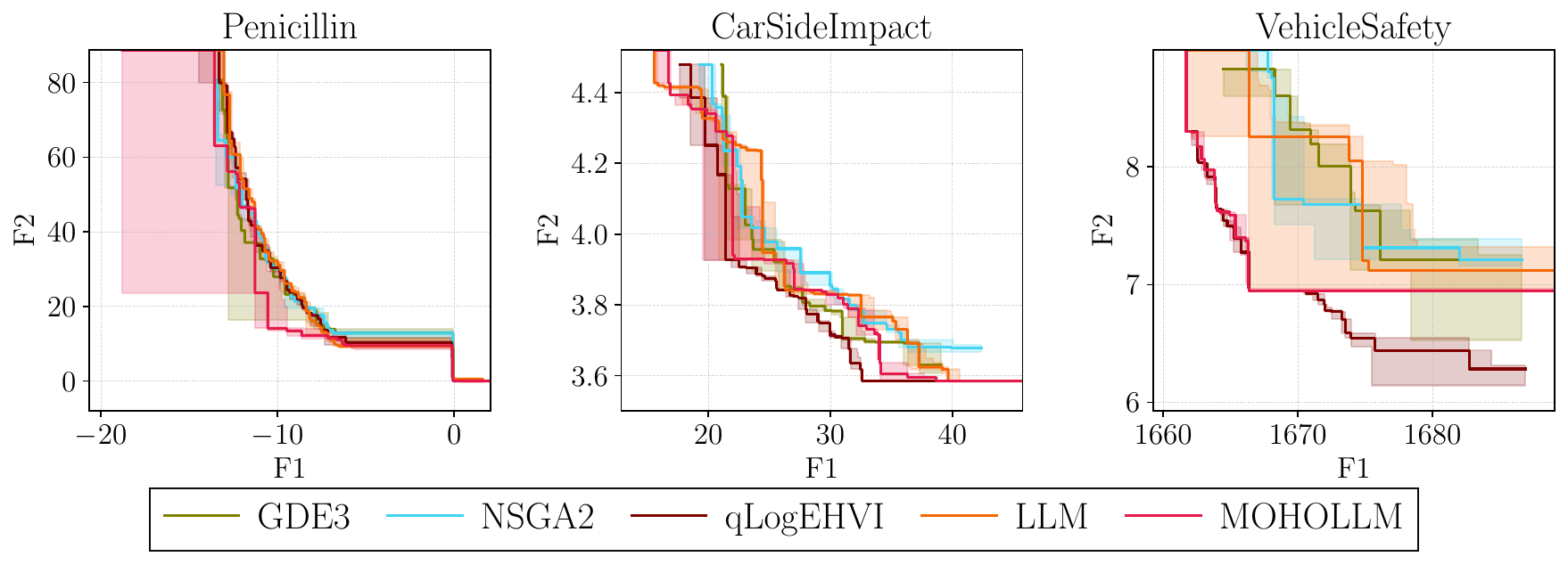}
    \caption{Empirical Attainment Function (EAF) on the real-world benchmarks (Penicillin, CarSideImpact, and VehicleSafety). For each problem, the plots show the EAF aggregated over 10 independent runs, comparing MOHOLLM against representative baselines (GDE3, NSGA-II, qLogEHVI, and the global LLM). Shaded regions indicate the variance across runs.}
    \label{fig:eaf}
\end{figure*}


\begin{figure*}[ht]
    \centering
    \includegraphics[width=0.99\linewidth]{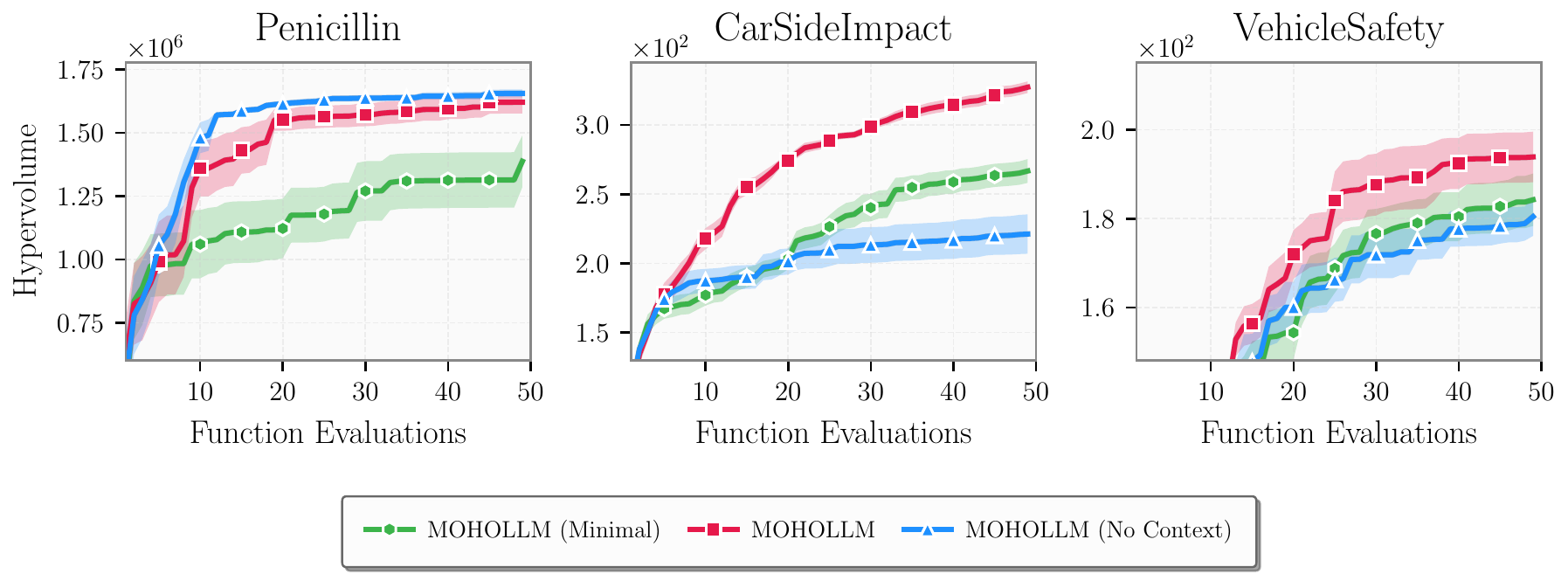}
    \caption{Prompt ablations on real-world benchmarks (Penicillin, VehicleSafety, and CarSideImpact). We compare the full MOHOLLM configuration against variants obtained by removing or modifying prompt components.}
    \label{fig:hv_real_prompt_ablation_all}
\end{figure*}

\begin{figure*}[ht]
    \centering
    \includegraphics[width=0.99\linewidth]{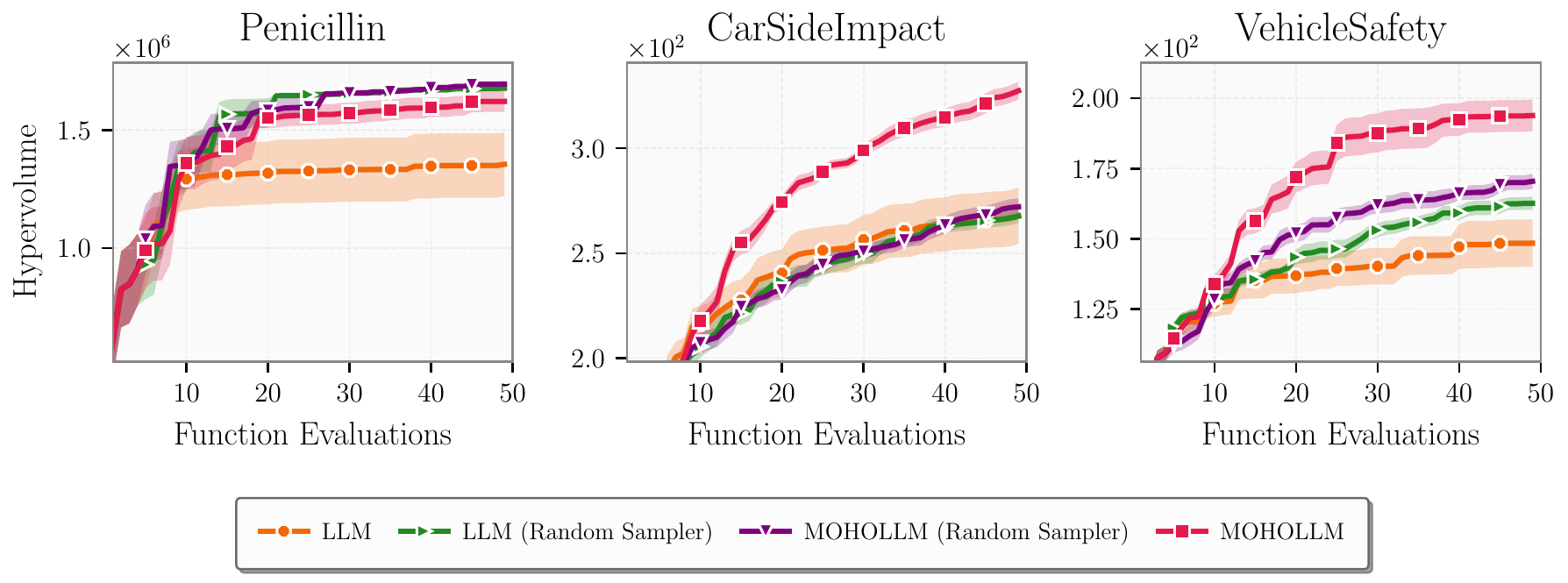}
    \caption{Candidate sampler ablation comparing MOHOLLM against variants with a random candidate sampler.}
    \label{fig:sampler_ablation}
\end{figure*}

\begin{figure*}[ht]
    \centering
    \includegraphics[width=0.99\linewidth]{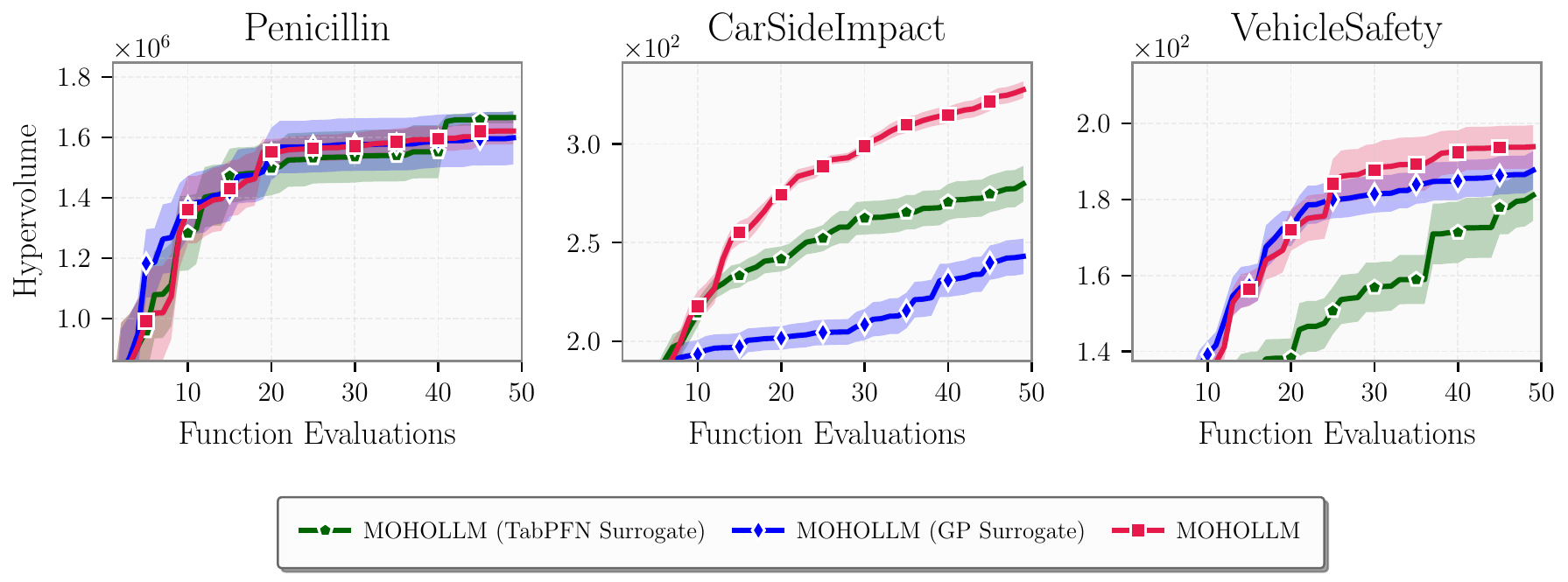}
    \caption{Surrogate model ablation on the real-world benchmarks (Penicillin, CarSideImpact, and VehicleSafety), comparing MOHOLLM with different surrogate choices (TabPFN surrogate and GP surrogate) against the default LLM surrogate. We report the mean hypervolume trajectories over 10 runs with $95\%$ confidence intervals.}
    \label{fig:surr_ablation}
\end{figure*}

\begin{figure*}[ht]
    \centering
    \includegraphics[width=0.99\linewidth]{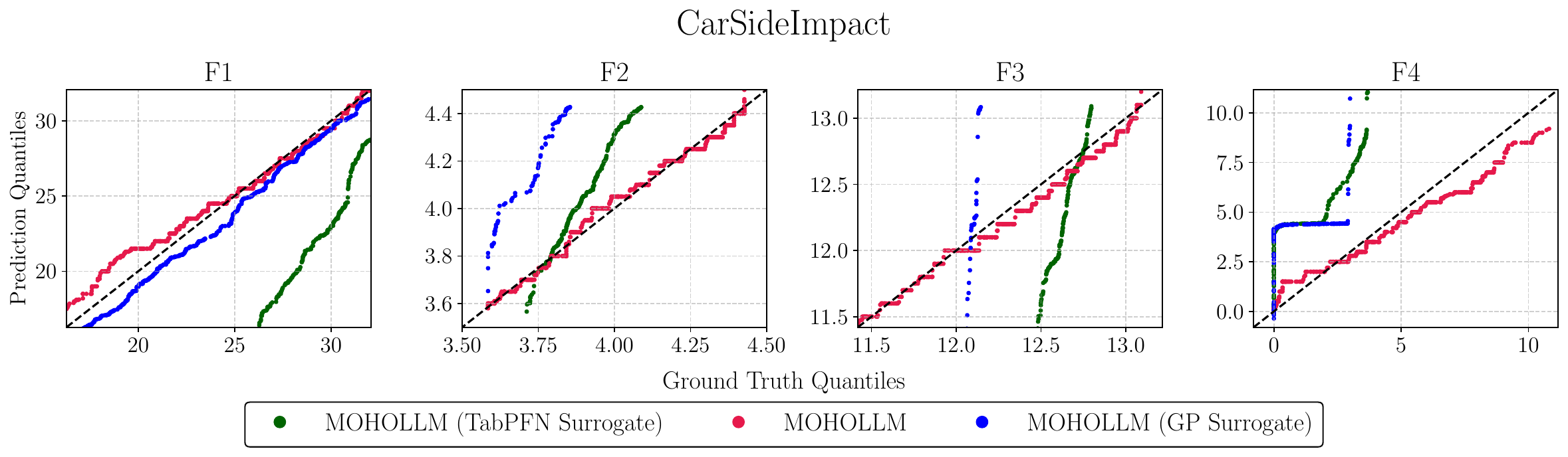}
    \caption{Q-Q plots on CarSideImpact comparing surrogate predictions against ground-truth quantiles for each objective (F1–F5), illustrating calibration and distributional alignment of the TabPFN, GP, and default LLM surrogates.}
    \label{fig:qq-carsideimpact}
\end{figure*}

\begin{figure*}[ht]
    \centering
    \includegraphics[width=0.99\linewidth]{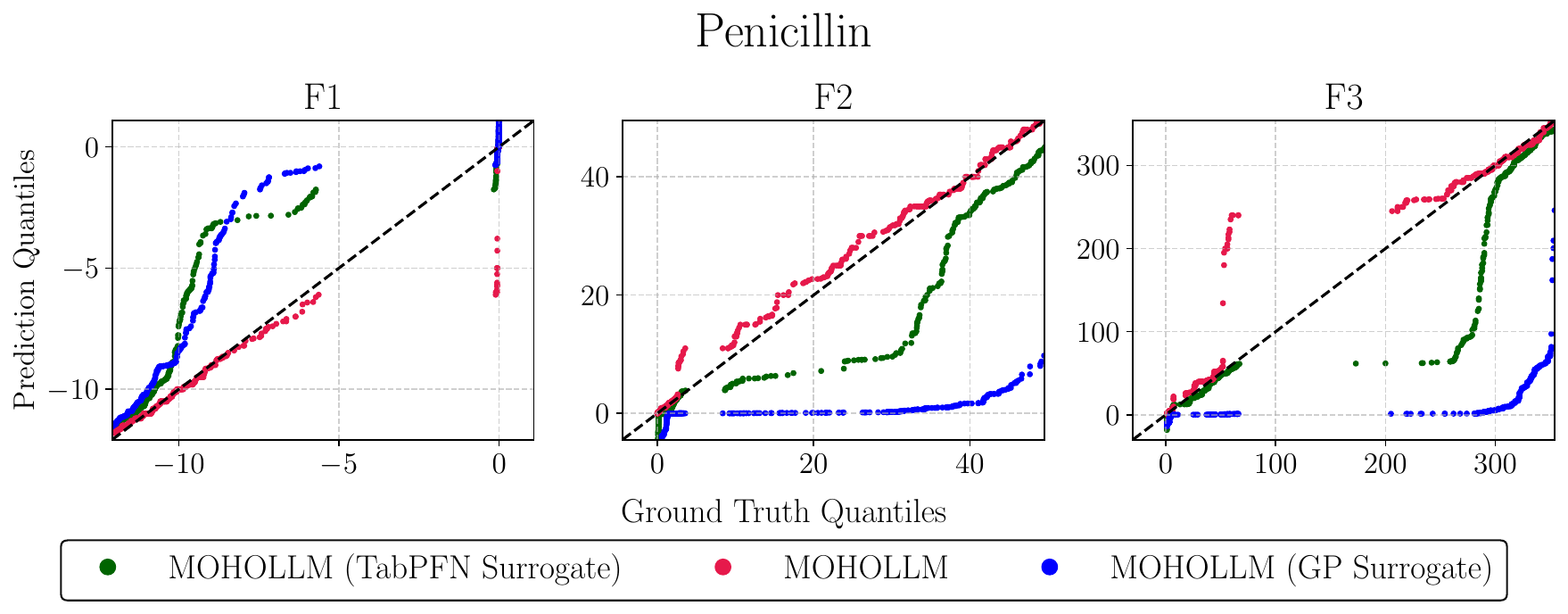}
    \caption{Q-Q plots on Penicillin analogue to Figure~\ref{fig:qq-carsideimpact}.}
    \label{fig:qq-penicillin}
\end{figure*}

\begin{figure*}[ht]
    \centering
    \includegraphics[width=0.99\linewidth]{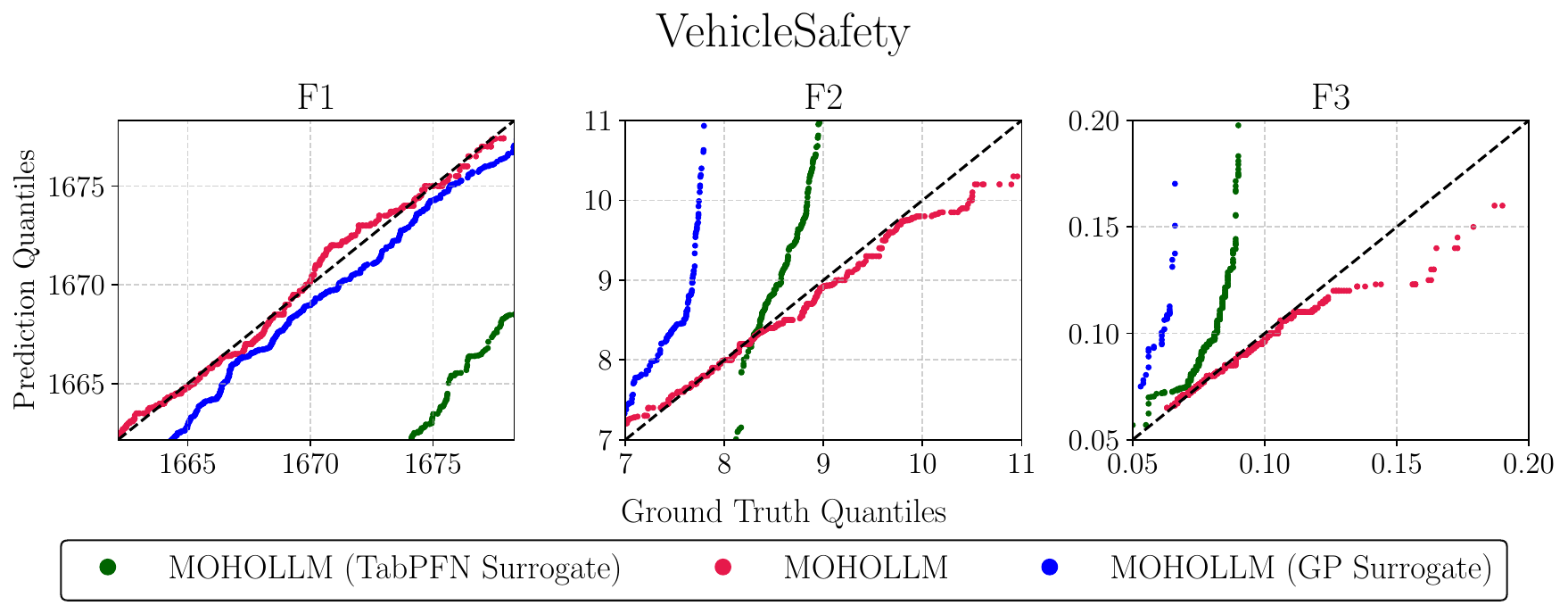}
    \caption{Q-Q plots on VehicleSafety analogue to Figure~\ref{fig:qq-carsideimpact}.}
    \label{fig:qq-vehiclesafety}
\end{figure*}


\begin{figure*}[ht]
    \centering
    \includegraphics[width=0.99\linewidth]{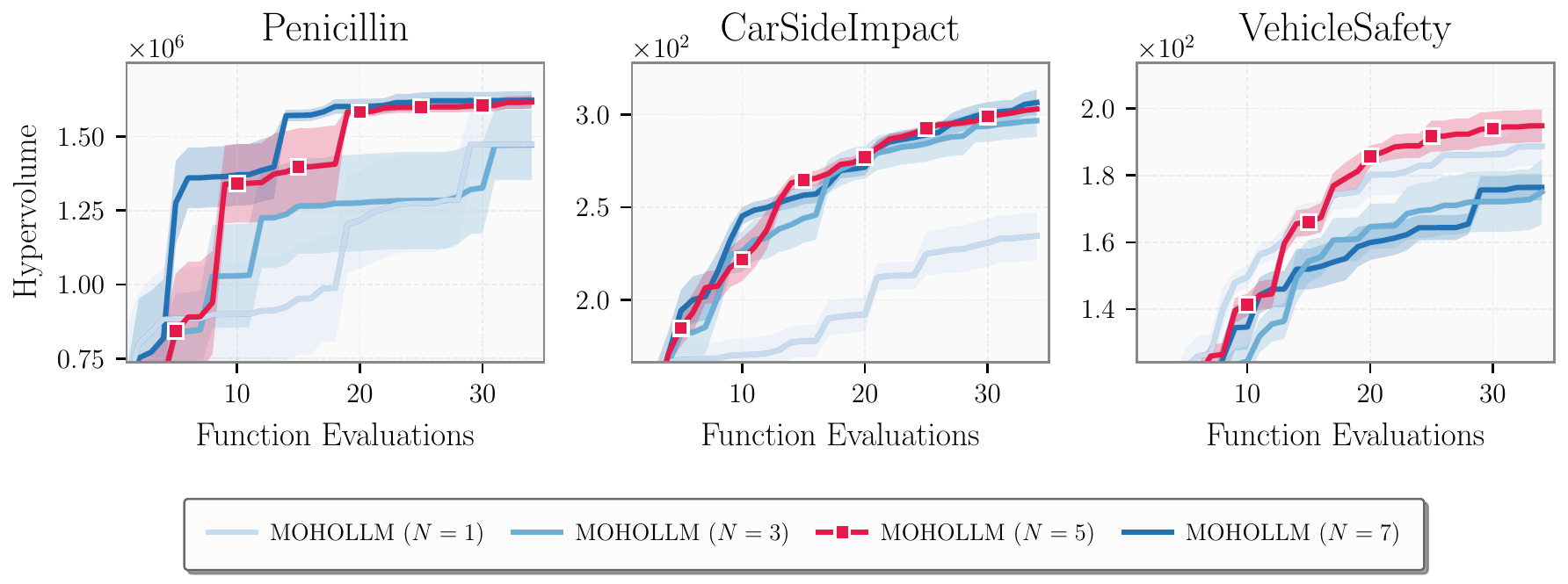}
    \caption{The effect of varying the number of candidates $N$ generated from the LLM per iteration on the real-world benchmarks. Higher $N$ typically improves convergence speed and final solution quality up to a saturation point.}
    \label{fig:ablation_k}
\end{figure*}

\begin{figure*}[ht]
    \centering
    \includegraphics[width=0.99\linewidth]{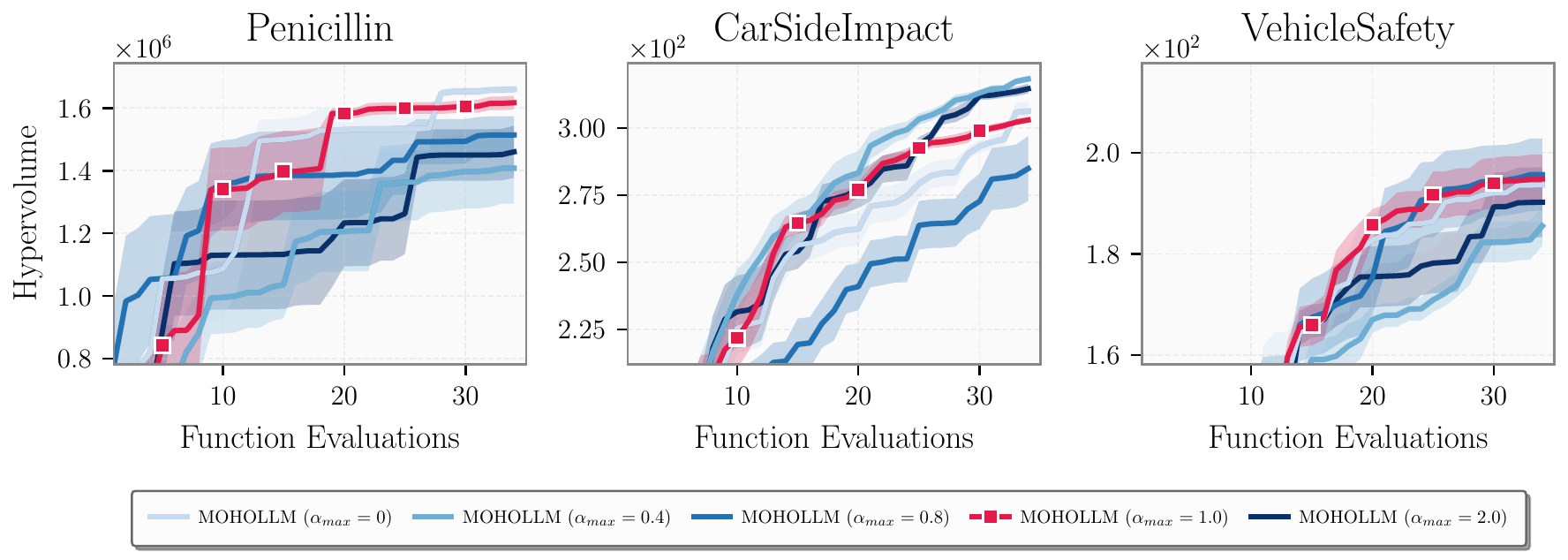}
    \caption{The effect of varying $\alpha_{max}$ on the real-world benchmarks.}
    \label{fig:ablation_alpha}
\end{figure*}

\begin{table}[ht]
\centering
\caption{Per-objective surrogate prediction accuracy on the CarSideImpact benchmark, reporting $R^2$, MSE, Kendall’s $\tau$, and Spearman’s $\rho$ over $N=700$ test points for each method and objective.}
\label{tab:surrogate_metrics_carside}
\resizebox{\linewidth}{!}{%
\begin{tabular}{llrrrrr}
\toprule
\textbf{Method} & \textbf{Obj} & $\boldsymbol{R^2}\,\uparrow$ & \textbf{MSE}$\,\downarrow$ & $\boldsymbol{\tau}\,\uparrow$ & $\boldsymbol{\rho}\,\uparrow$ & \textbf{N} \\
\midrule
\multirow{4}{*}{Gemini 2.0 Flash} 
 & F1 & \textbf{0.877} & \textbf{5.53} & \textbf{0.786} & \textbf{0.931} & 700 \\
 & F2 & \textbf{0.645} & \textbf{0.02} & \textbf{0.632} & \textbf{0.800} & 700 \\
 & F3 & \textbf{0.874} & \textbf{0.05} & \textbf{0.774} & \textbf{0.920} & 700 \\
 & F4 & \textbf{0.632} & \textbf{4.32} & \textbf{0.682} & \textbf{0.847} & 700 \\
\midrule
\multirow{4}{*}{GP} 
 & F1 & 0.010 & 43.35 & 0.781 & 0.802 & 700 \\
 & F2 & -6889.585 & 376.37 & -0.513 & -0.628 & 700 \\
 & F3 & -496.359 & 149.84 & -0.421 & -0.467 & 700 \\
 & F4 & -21.427 & 281.44 & -0.129 & -0.155 & 700 \\
\midrule
\multirow{4}{*}{TabPFN-v2} 
 & F1 & -5.616 & 242.86 & 0.061 & 0.037 & 700 \\
 & F2 & -3340.140 & 196.87 & 0.127 & 0.140 & 700 \\
 & F3 & -287.800 & 88.74 & 0.116 & 0.140 & 700 \\
 & F4 & -15.547 & 204.81 & 0.107 & 0.154 & 700 \\
\bottomrule
\end{tabular}%
}
\end{table}

\begin{table}[ht]
\centering
\caption{Per-objective surrogate prediction accuracy on the Penicillin benchmark, reporting $R^2$, MSE, Kendall’s $\tau$, and Spearman’s $\rho$ over $N=700$ test points for each method and objective.}
\label{tab:surrogate_metrics_penicillin}
\resizebox{\linewidth}{!}{%
\begin{tabular}{llrrrrr}
\toprule
\textbf{Method} & \textbf{Obj} & $\boldsymbol{R^2}$ & \textbf{MSE} & $\boldsymbol{\tau}$ & $\boldsymbol{\rho}$ & \textbf{N} \\
\midrule
\multirow{3}{*}{Gemini 2.0 Flash} 
 & F1 & \textbf{0.742} & \textbf{8.06} & \textbf{0.684} & \textbf{0.851} & 700 \\
 & F2 & \textbf{0.792} & \textbf{122.67} & \textbf{0.797} & \textbf{0.896} & 700 \\
 & F3 & \textbf{0.787} & \textbf{6031.37} & \textbf{0.824} & \textbf{0.909} & 700 \\
\midrule
\multirow{3}{*}{GP} 
 & F1 & -9.721 & 314.59 & 0.274 & 0.371 & 700 \\
 & F2 & -0.769 & 1039.85 & -0.210 & -0.295 & 700 \\
 & F3 & -0.397 & 39389.47 & 0.028 & 0.065 & 700 \\
\midrule
\multirow{3}{*}{TabPFN-v2} 
 & F1 & -173.697 & 5184.01 & 0.375 & 0.487 & 700 \\
 & F2 & -0.537 & 987.24 & 0.392 & 0.475 & 700 \\
 & F3 & 0.491 & 11916.89 & 0.573 & 0.720 & 700 \\
\bottomrule
\end{tabular}%
}
\end{table}

\begin{table}[ht]
\centering
\caption{Per-objective surrogate prediction accuracy on the VehicleSafety benchmark, reporting $R^2$, MSE, Kendall’s $\tau$, and Spearman’s $\rho$ over $N=700$ test points for each method and objective.}
\label{tab:surrogate_metrics_vehicle}
\resizebox{\linewidth}{!}{%
\begin{tabular}{llrrrrr}
\toprule
\textbf{Method} & \textbf{Obj} & $\boldsymbol{R^2}$ & \textbf{MSE} & $\boldsymbol{\tau}$ & $\boldsymbol{\rho}$ & \textbf{N} \\
\midrule
\multirow{3}{*}{Gemini 2.0 Flash} 
 & F1 & \textbf{0.806} & \textbf{7.57} & \textbf{0.758} & \textbf{0.903} & 700 \\
 & F2 & \textbf{0.655} & \textbf{0.25} & \textbf{0.672} & \textbf{0.805} & 700 \\
 & F3 & \textbf{0.524} & \textbf{0.00} & \textbf{0.655} & \textbf{0.813} & 700 \\
\midrule
\multirow{3}{*}{GP} 
 & F1 & -9038.058 & 402521.05 & 0.709 & 0.719 & 700 \\
 & F2 & -2671946.160 & 2216819.87 & 0.257 & 0.345 & 700 \\
 & F3 & -1789516768.903 & 1952829.18 & 0.056 & 0.069 & 700 \\
\midrule
\multirow{3}{*}{TabPFN-v2} 
 & F1 & -16837.928 & 1272547.79 & 0.286 & 0.306 & 700 \\
 & F2 & -1349159.150 & 1365562.21 & 0.221 & 0.274 & 700 \\
 & F3 & -1273121958.222 & 1111401.55 & 0.190 & 0.234 & 700 \\
\bottomrule
\end{tabular}%
}
\end{table}

\clearpage

\begin{figure*}[ht]
    \centering
    \includegraphics[width=0.99\linewidth]{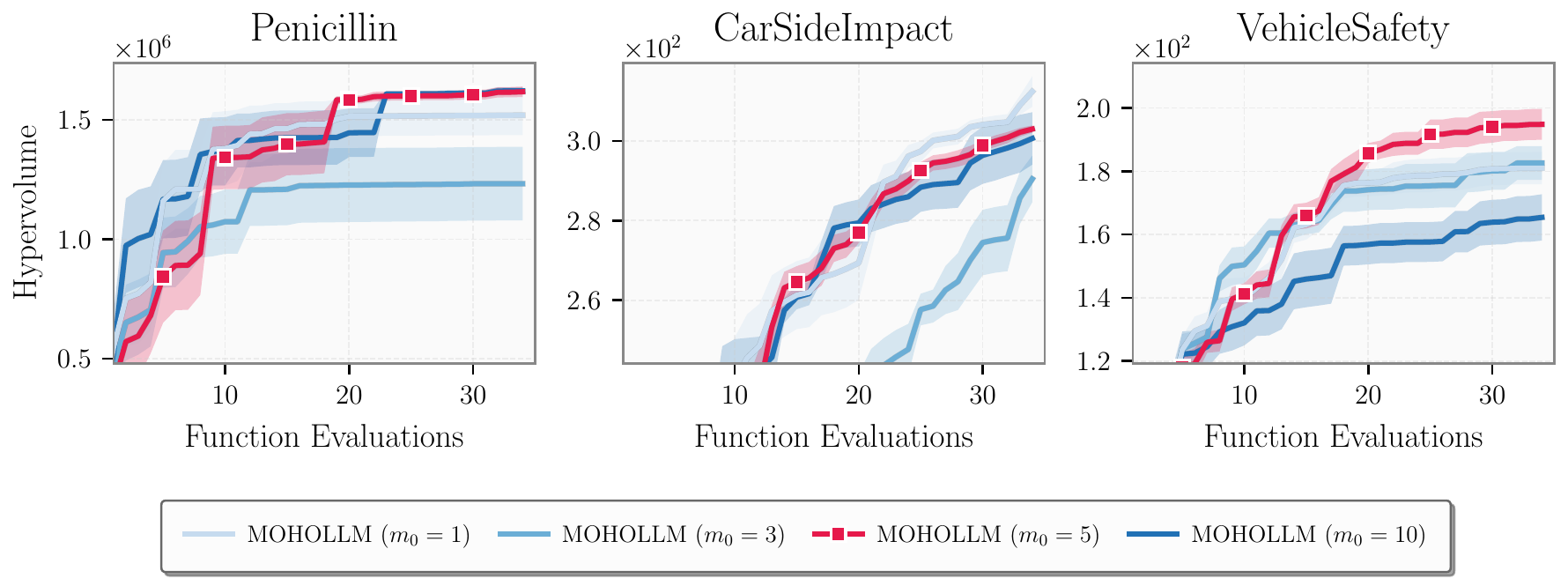}
    \caption{The effect of varying the initial maximum leaf size $m_0$ on the real-world benchmarks. $m_0 = 5$ is a robust choice.}
    \label{fig:ablation_m}
\end{figure*}

\begin{figure*}[ht]
    \centering
    \includegraphics[width=0.99\linewidth]{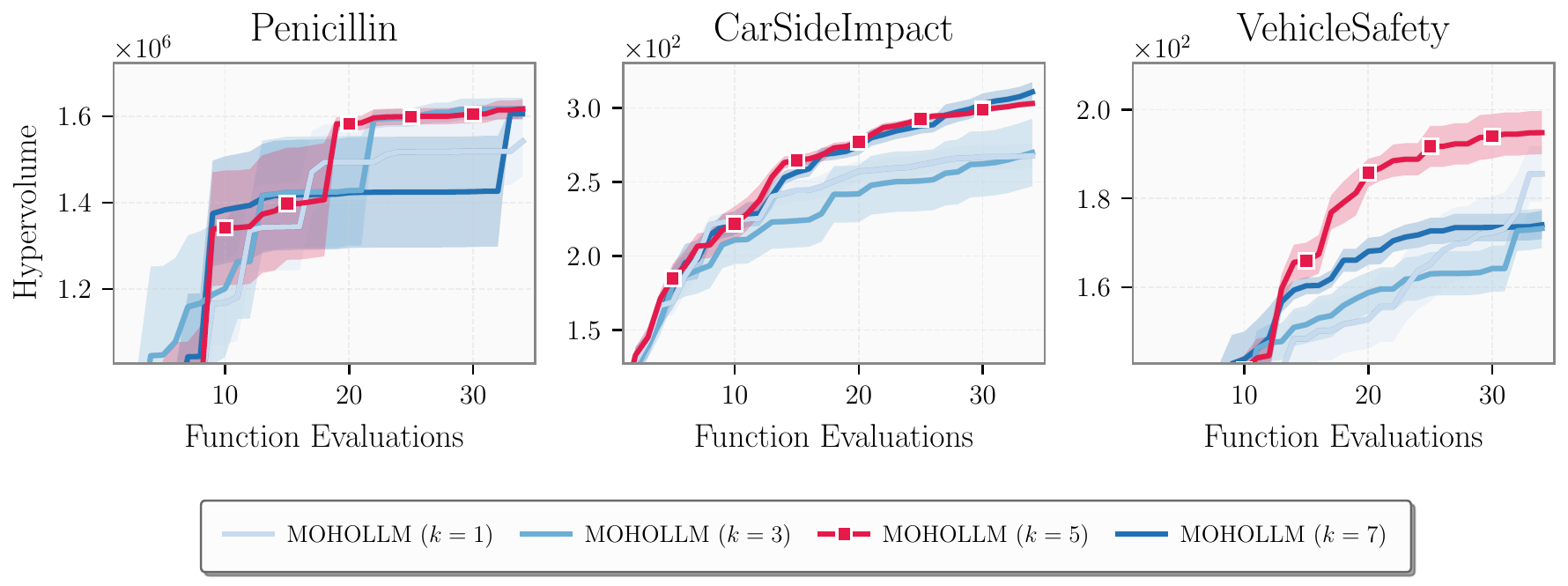}
    \caption{The effect of varying the number of sampled regions $k$ on the real-world benchmarks.}
    \label{fig:ablation_M}
\end{figure*}

\begin{table*}[ht]
\centering
\caption{Definition of the dynamic prompt variables used in MOHOLLM.}
\label{tab:placeholders}
\small
\resizebox{\textwidth}{!}{%
\begin{tabular}{@{}l p{5cm} p{6cm}@{}}
\toprule
\textbf{Placeholder} & \textbf{Description} & \textbf{Example Content} \\
\midrule
\texttt{\$metrics} & The specific objective functions to be optimized. & \texttt{F1 (lower is better), F2 (lower is better)} \\
\midrule
\texttt{\$region\_constraints} & The hard bounds of the current hyperrectangular partition (leaf node). & 
\texttt{\{lr: range(float([0.0, 0.9])), \newline num\_layer: range(int([1, 20]))\}} \\
\midrule
\texttt{\$region\_ICL\_examples} & In-Context Learning examples sampled from the history $\mathcal{D}_t$ to guide generation. & 
\texttt{\{lr: 0.4, num\_layer: 8\} \newline F1: 5.65} \\
\midrule
\texttt{\$target\_number\_of\_candidates} & The number of candidates ($M$) requested from the sampler. & \texttt{15} \\
\midrule
\texttt{\$candidate\_sampler\_response\_format} & The required JSON schema for output parsing. & \texttt{\{lr: float, num\_layer: int\}} \\
\midrule
\texttt{\$target\_architectures} & The set of newly generated candidates requiring surrogate scoring. & 
\texttt{1: \{lr: 0.4, ...\} \newline 2: \{lr: 0.03, ...\}} \\
\midrule
\texttt{\$surrogate\_model\_response\_format} & The required JSON schema for predicted values. & \texttt{\{F1: float, F2: float\}} \\
\bottomrule
\end{tabular}
}
\end{table*}

\clearpage
\newpage
\section{Prompt Engineering and Templates}
\label{appendix:prompt_design}

\subsection{Generic Prompt Templates}
To enable the LLM to act as both a candidate sampler and a surrogate model, we employ structured prompt templates. These templates are dynamically populated at each iteration with task-specific constraints, optimization history, and performance metrics. We visually distinguish the templates by function: \textcolor{blue}{blue} for candidate sampling and \textcolor{green}{green} for surrogate modeling. Listing~\ref{lst:prompt_template_candidate_sampler} and \ref{lst:prompt_template_surrogate_model} present the main templates.

\begin{lstlisting}[style=sampler, caption={Generic prompt template for the candidate sampler.}, label={lst:prompt_template_candidate_sampler}]
# Optimization task

## Problem Description
You are tasked with solving a optimization problem that requires finding optimal solutions.

- **Evaluation**: Configurations are measured by $metrics

$description

## Constraints
The allowable ranges for the hyperparameters are:
$input_boundaries

## Previously Evaluated configurations
Below are examples of configurations that have been evaluated, showing their operations and performance metrics:

$ICL_examples

## Your Task
Generate $target_number_of_candidates new configurations that:
1. Are likely to achieve lower $metrics than the examples
2. Are different from all previously evaluated configurations
3. Satisfy all the specified constraints: $input_boundaries

## Output Format
Each configuration has to follow this format:

$candidate_sampler_response_format

Provide your response in a JSON list containing each proposed configuration.
Return only the required JSON list output without additional text.
\end{lstlisting}

\begin{lstlisting}[style=surrogate, caption={Generic prompt template for the surrogate Mmdel.}, label={lst:prompt_template_surrogate_model}]
# Configuration Performance Prediction

## Problem Description
You are tasked with predicting the performance of configurations.

- **Evaluation Metric**: $metrics (to be predicted)
- **Constraint**: The allowable ranges for the hyperparameter are: $input_boundaries

$description

## Reference configurations with Known Performance
Below are examples of configurations that have been evaluated, showing their operations and performance metrics:

$ICL_examples

## Candidate configurations to Evaluate
You must predict performance for these new configurations:

$target_architectures

## Your Task
1. Predict the $metrics value for each candidate configuration
2. Base your predictions on patterns in the reference examples

## Output Format
Each evaluation has to follow this format:

$surrogate_model_response_format

Provide your response in a JSON list containing each proposed evaluation.
Return only the required JSON list output without additional text.
\end{lstlisting}

\subsection{Prompt Variable Definitions}
\label{appendix:prompt_template_tag_description}

Table~\ref{tab:placeholders} details the dynamic placeholders used within the templates. These tags are programmatically replaced at runtime based on the current leaf regions' boundaries and the optimization history.

\subsection{Prompt Ablation Studies}
\label{appendix:prompt_ablation}
The following listings shows the \texttt{Minimal} templates used to isolate the impact of prompt complexity.

\begin{lstlisting}[style=sampler, caption={Minimal candidate sampler prompt template (ablation).}, label={lst:ablation_sampler}]
Previously Evaluated Configurations:
$Region_ICL_examples

Generate $target_number_of_candidates new configurations that:
1. Are likely to achieve lower $metrics than the examples
2. Are different from all previously evaluated configurations
3. Satisfy all the specified constraints: $region_constraints

Format:
Each configuration has to follow this format:

$candidate_sampler_response_format

Provide your response in a JSON list containing each proposed configuration.
Return only the required JSON list output without additional text.
\end{lstlisting}

\begin{lstlisting}[style=surrogate, caption={Minimal surrogate prompt template (ablation).}, label={lst:ablation_surrogate}]
Previously Evaluated Configurations:
$Region_ICL_examples

Predict performance for these configurations:
$target_architectures

Task:
1. Predict the $metrics value for each candidate configuration
2. Base your predictions on patterns in the reference examples

Format:
Each evaluation has to follow this format:

$surrogate_model_response_format

Provide your response in a JSON list containing each proposed evaluation.
Return only the required JSON list output without additional text.
\end{lstlisting}

\subsection{Domain-Specific Contexts}
\label{appendix:additional_context}
For real-world benchmarks, we inject specific problem descriptions into the \texttt{\$description} tag to aid the LLM's semantic reasoning.

\begin{lstlisting}[style=context, caption={Context for the VehicleSafety benchmark.}, label={lst:context_vehicle_safety}]
Your task is to optimize the design of a vehicle for frontal crash safety by adjusting the material thickness of five key structural components. You will generate a configuration for: x0, the bumper beam that absorbs initial impact; x1, the crash box designed to crush progressively; x2, the main longitudinal rails that channel energy; x3, the A-pillar that protects the cabin integrity; and x4, the dash panel that prevents intrusion into the legroom area. The performance of your design will be evaluated against three competing objectives to be minimized: F1 is the total vehicle mass, F2 is the chest injury criterion, and F3 is the toe board intrusion. Your goal is to propose designs that find the best trade-off between minimizing weight, occupant injury, and structural deformation.
\end{lstlisting}

\begin{lstlisting}[style=context, caption={Context for the CarSideImpact benchmark.}, label={lst:context_car_side_impact}]
Your task is to propose optimal designs for a vehicle to improve its safety in a side-impact collision. You will generate a configuration of seven input variables representing the thickness of key structural components: x0 (B-Pillar inner), x1 (B-Pillar reinforcement), x2 (floor side inner), x3 (cross-member), x4 (door beam), x5 (door beltline reinforcement), and x6 (roof rail). The performance of your design will be judged on four objectives, all of which should be minimized: F1 is the vehicle's total weight, F2 is the injury load on the occupant's abdomen, F3 is the intrusion velocity at key points, and F4 is a penalty for any constraint violations. Your goal is to find designs that represent the best possible trade-offs across these competing safety and engineering metrics.
\end{lstlisting}

\begin{lstlisting}[style=context, caption={Context for the Penicillin benchmark.}, label={lst:context_penicillin}]
Your task is to find the optimal settings for a simulated fed-batch penicillin production process. You will generate a configuration of seven input control parameters that define the initial conditions and operation of the fermenter. These parameters are: x0, the culture medium volume; x1, the biomass concentration; x2, the operating temperature; x3, the glucose substrate concentration; x4, the substrate feed rate; x5, the substrate feed concentration; and x6, the H+ concentration (acidity). The success of your configuration is evaluated on three metrics in a multi-objective optimization context, all of which are to be minimized. F1 represents the negative final penicillin yield, F2 is the total production time, and F3 is the total CO2 emission byproduct. Your goal is to propose configurations that find the best trade-offs by minimizing all three competing objectives.
\end{lstlisting}

\subsection{Instantiated Examples (Poloni Benchmark)}
To illustrate how the templates are populated during runtime, the following listings show the exact prompts sent to the LLM for the \textit{Poloni} benchmark.

\begin{lstlisting}[style=sampler, caption={Example candidate sampler prompt used in the Poloni benchmark.}, label={lst:example_sampler_poloni}]
# Optimization task

## Problem Description
You are tasked with solving a optimization problem that requires finding optimal solutions.

- **Evaluation**: Configurations are measured by F1 (lower is better), F2 (lower is better)

## Constraints
The allowable ranges for the hyperparameters are:
{
  x: range(float([-3.142, -1.067])),
  y: range(float([0.274, 3.142]))
}

## Previously Evaluated Configurations
Below are examples of configurations that have been evaluated, showing their operations and performance metrics:

Configuration: {"x": -3.142, "y": -0.725}
F1: 13.327, F2: 0.096

Configuration: {"x": -3.142, "y": -0.706}
F1: 13.178, F2: 0.107

Configuration: {"x": -3.142, "y": -0.162}
F1: 7.839, F2: 0.722

Configuration: {"x": 0.975, "y": 1.681}
F1: 1.18, F2: 22.988

Configuration: {"x": -3.142, "y": -0.82}
F1: 14.006, F2: 0.053

Configuration: {"x": -3.142, "y": 0.494}
F1: 2.616, F2: 2.252

Configuration: {"x": -3.142, "y": 0.017}
F1: 6.028, F2: 1.054

Configuration: {"x": -3.142, "y": -0.231}
F1: 8.566, F2: 0.612

Configuration: {"x": -3.142, "y": -0.392}
F1: 10.257, F2: 0.39

## Your Task
Generate 5 new configurations that:
1. Are likely to achieve lower F1 (lower is better), F2 (lower is better) than the examples
2. Are different from all previously evaluated configurations
3. Satisfy all the specified constraints: {
  x: range(float([-3.142, -1.067])),
  y: range(float([0.274, 3.142]))
}

## Output Format
Each configuration has to follow this format:

{"x": $x, "y": $y}

Provide your response in a JSON list containing each proposed configuration.
Return only the required JSON list output without additional text.
\end{lstlisting}

\begin{lstlisting}[style=surrogate, caption={Example surrogate model prompt used in the Poloni benchmark.}, label={lst:example_surrogate_poloni}]
# Configuration Performance Prediction
## Problem Description
You are tasked with predicting the performance of configurations.

- **Evaluation Metric**: F1 (lower is better), F2 (lower is better) (to be predicted)

## Reference configurations with Known Performance
Below are examples of configurations that have been evaluated, showing their operations and performance metrics:

Configuration: {"x": -3.142, "y": -0.725}
F1: 13.327, F2: 0.096

Configuration: {"x": -3.142, "y": -0.706}
F1: 13.178, F2: 0.107

Configuration: {"x": -3.142, "y": -0.162}
F1: 7.839, F2: 0.722

Configuration: {"x": 0.975, "y": 1.681}
F1: 1.18, F2: 22.988

Configuration: {"x": -3.142, "y": -0.82}
F1: 14.006, F2: 0.053

Configuration: {"x": -3.142, "y": 0.494}
F1: 2.616, F2: 2.252

Configuration: {"x": -3.142, "y": 0.017}
F1: 6.028, F2: 1.054

Configuration: {"x": -3.142, "y": -0.231}
F1: 8.566, F2: 0.612

Configuration: {"x": -3.142, "y": -0.392}
F1: 10.257, F2: 0.39

## Candidate configurations to Evaluate
You must predict performance for these new configurations:

1: {'x': -2.742, 'y': 0.474}
2: {'x': -2.342, 'y': 0.674}
3: {'x': -2.542, 'y': 0.574}
4: {'x': -2.942, 'y': 0.374}
5: {'x': -2.142, 'y': 0.774}

## Your Task
1. Predict the F1 (lower is better), F2 (lower is better) value for each candidate configuration
2. Base your predictions on patterns in the reference examples

## Output Format
Each evaluation has to follow this format:

{"F1": $F1, "F2": $F2}

Provide your response in a JSON list containing each proposed evaluation.
Return only the required JSON list output without additional text.
\end{lstlisting}

\end{document}

%% file: method.tex
\section{The MOHOLLM Algorithm}\label{sec:method}

We now introduce MOHOLLM, a method that solves blackbox multi-objective optimization problems with LLMs and search space partitioning.
By decomposing the continuous search space into discrete regions, we can apply discrete decision-making strategies to balance exploration and exploitation. The continuous optimization problem can hence be seen as a hierarchical MAB problem~\cite{valko13,grill15}. Similarly to \cite{schwanke2025improving}, our algorithm (see Algorithm~\ref{alg:mohollm} for the pseudocode) consists of five iterative steps:

\paragraph{1. Adaptive Space Partitioning.}
To partition the continuous domain $\mathcal{X}$ into regions, we employ an adaptive discretization strategy based on KD-trees~\cite{munos11,valko13}. One reason we chose axis-aligned partitioning is because of the easier interpretation of the hyperrectangular region boundaries from the LLM in step 4. At each iteration $t$, we construct the tree $\mathcal{T}_t$ recursively from the observation history $\mathcal{D}_t = \{(x_i, \mathbf{F}(x_i))\}_{i=1}^t$, where partitions are defined by splitting along the dimension of maximal variance at the median value. This median-based approach ensures a balanced tree structure, avoiding degenerate empty leaf nodes and positioning less observations on region boundaries. To prevent the partitions from becoming too fine-grained, especially in early iterations, we control the granularity of the tree via a dynamic splitting criterion: a leaf node is refined only if its sample count exceeds a time-dependent threshold $m_t = m_0 + \lfloor \lambda \log(1+t) \rfloor$. This schedule serves as a structural regularizer, constraining the tree depth in early iterations to force broad, global exploration, while still allowing the partition diameter to vanish asymptotically.

\begin{algorithm}[t]
\caption{MOHOLLM Algorithm Pseudocode}
\label{alg:mohollm}
\begin{algorithmic}[1]
\REQUIRE Initial data $\mathcal{D}_{t_0}$, budget $T$, batch size $b$, regions $k$, candidates $N$, initial leaf size $m_0$.
\STATE Initialize $t \leftarrow t_0$.
\WHILE{$t < T$}
    \STATE \textbf{// Step 1: Partitioning}
    \STATE Construct KD-tree $\mathcal{T}_t$ on $\mathcal{D}_t$
    \STATE Adjust leaf maximum size $m_t = m_0 + \lfloor \lambda \log(1+t) \rfloor$
    
    \STATE \textbf{// Step 2: Scoring}
    \FOR{each leaf $R_j$}
        \STATE Compute Exploitation: $\psi_{\text{HV}}(R_j)$ (Eq.~\ref{eq:exploit})
        \STATE Compute Exploration: \\ \quad $\psi_{\text{Vol}}(R_j)$ (Eq.~\ref{eq:volume}) \& $\psi_{\text{UCBV}}(R_j)$ (Eq.~\ref{eq:ucb})
        \STATE Compute composite score $A_j(t)$ (Eq.~\ref{eq:score})
    \ENDFOR

    \STATE \textbf{// Step 3: Selection}
    \STATE Compute probabilities of regions $\boldsymbol{\pi}(t) \propto \exp(\mathbf{A}(t))$
    \STATE $\mathcal{L}^{sampled}_{t} \leftarrow$ sample $k$ regions from $\mathtt{Cat}(\boldsymbol{\pi}(t))$
    
    \STATE \textbf{// Step 4: Sampling}
    \FOR{each $R \in \mathcal{L}^{sampled}_{t}$}
        \STATE Generate $N$ candidates $\hat{x} \sim p_{\text{LLM}}(x \mid \mathcal{C}(R_j, \mathcal{D}_t))$
    \ENDFOR
    
    \STATE \textbf{// Step 5: Evaluation}
    \STATE For all candidates predict $\hat{\mathbf{y}} \sim p_{\text{LLM}}(\mathbf{y} \mid \mathcal{C}(\mathcal{D}_t,\hat{x}))$
    \STATE Select batch $B^*$ of size $b$ maximizing predicted HV
    \STATE Evaluate true objectives: $\mathbf{Y}^{*} \leftarrow \mathbf{F}(B^*)$
    \STATE Update $\mathcal{D}_{t} \leftarrow \mathcal{D}_t \cup (B^*, \mathbf{Y}^{*})$
    \STATE $t \leftarrow t + b$
\ENDWHILE
\RETURN $\mathcal{F}_{final}$ (Pareto Front of $\mathcal{D}_T$)
\end{algorithmic}
\end{algorithm}

\begin{figure*}[ht]
    \centering
    \includegraphics[width=0.99\linewidth]{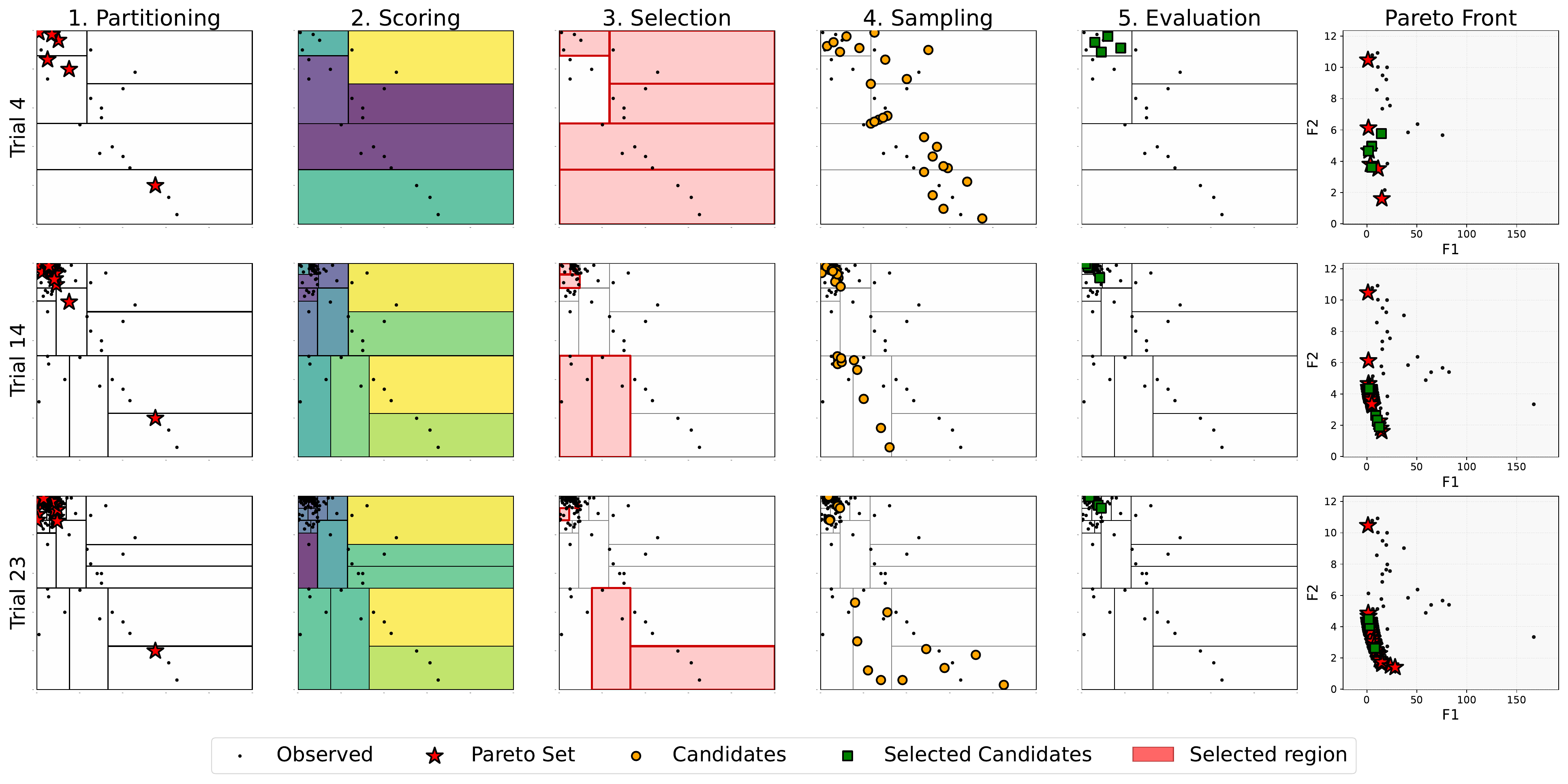}
    \caption{Search process of MOHOLLM on the 2D Branin–Currin function at different optimization stages. The last column illustrates the function values, while other columns the five algorithmic steps: partitioning, region scoring, probabilistic selection, LLM-based sampling, and evaluation. Rows correspond to early, intermediate, and late trials. MOHOLLM transitions from coarse, globally exploratory partitions to fine-grained, localized refinement around the Pareto front, allocating samples to both sparse regions and dense Pareto-optimal areas.}
    \label{fig:search_process_branin}
\end{figure*}

\paragraph{2. Region Scoring.}
Let $\mathcal{L}_t = \{R_1, \dots, R_{K_t}\}$ denote the resulting set of $K_t$ disjoint leaf nodes of the KD-tree at iteration $t$\footnote{We write $R_j$ instead of $R_{t, j}$ for brevity.}. To decide on which region to sample next with the LLM, we need to score each region with an utility function, that balances exploration and exploitation, so we do not spend unnecessary resources in uninteresting or suboptimal regions.
To this end, we define a scalar utility function $A_j(t)$ that ranks each partition $R_j \in \mathcal{L}_t$ based on its (i) contribution to the current Pareto front HV, (ii) geometric properties, and (iii) statistical properties.
\begin{description}
    \item[i.] \textit{Pareto Exploitation ($\psi_{\textnormal{HV}}$):} We measure local utility via the \textit{Regional Hypervolume Contribution}\footnote{To ensure scale invariance, all objective vectors $\mathbf{y} \in \mathcal{R}^M$ are min–max normalized to $[0,1]$ before HV computation: $\tilde{\mathbf{y}} = (\mathbf{y} - \mathbf{y}_{\min}) / (\mathbf{y}_{\max} - \mathbf{y}_{\min})$, with bounds from the current history.}, which quantifies the total contribution that the points of each region $I_j \subset R_j$ have on the current Pareto front $\mathcal{F}^*_t$:
    \begin{equation}
        \label{eq:exploit}
        \psi_{\text{HV}}(R_j) = \text{HV}(\mathcal{F}^*_t) - \text{HV}(\mathcal{F}^*_t \setminus \{\, \mathbf{F}(x) \mid x \in I_j \,\}). 
    \end{equation}
    This metric aligns the search explicitly with the multiple optimization objectives, prioritizing regions that currently sustain the non-dominated set. $\psi_{\text{HV}}(R_j) = 0$ and $\psi_{\text{HV}}(R_j) = 1$ mean that no and all points from region $R_j$ are in the current Pareto set, respectively.

    \item[ii.] \textit{Geometric Exploration ($\psi_{\textnormal{Vol}}$):} Our first exploration term involves a "void-filling" incentive based on the geometric mean of the side lengths of each leaf hyperrectangular region:
    \begin{equation}
        \label{eq:volume}
        \psi_{\text{Vol}}(R_j) = \mathrm{vol}(R_j)^{1/d},
    \end{equation}
    where $d$ is the input dimensionality. $\psi_{\text{Vol}}$ ensures that large, under-sampled subspaces retain non-negligible selection probability purely due to their size, thereby preventing the premature exclusion of unvisited regions.
    
    \item[iii.] \textit{Statistical Exploration ($\psi_{\textnormal{UCBV}}$):} To assign high importance to regions which have high objective variance, we employ a variance-based Upper Confidence Bound~\cite{audibert09} as a second exploration term:
    \begin{equation}
        \label{eq:ucb}
        \psi_{\text{UCBV}}(R_j) = \sqrt{\frac{2\varsigma_j^2 \max(0, \ln(\frac{t}{K_t \cdot |I_j|}))}{|I_j|}},
    \end{equation}
    where $K_t$ is the number of active leaves, $|I_j|$ is the number of points in region $R_j$ and $\varsigma_j^2 = Var(\{\text{HV}(\mathcal{F}^*_t) - \text{HV}(\mathcal{F}^*_t \setminus  \mathbf{F}(x) )\}_{x \in I_j})$ is the variance contribution of each point to the current Pareto front HV.
\end{description}
The final score aggregates these components via a non-linear scalarization:
\begin{align}
    \label{eq:score}
    A_j(t) = & \sigma(\psi_{\text{HV}}(R_j)) \\ & + \alpha_t \left[ \beta_1 \sigma(\psi_{\text{Vol}}(R_j)) + \beta_2 \sigma(\psi_{\text{UCBV}}(R_j)) \right], \nonumber
\end{align}
where $\beta_1 + \beta_2 = 1$. The sigmoid function $\sigma(\cdot)$ normalizes the different ranges from each component into a common $[0,1]$ interval. $\alpha_t$ is a temperature parameter that decays following a cosine annealing schedule as $t$ increases, therefore suppressing the exploration terms ($\psi_{\text{Vol}}, \psi_{\text{UCBV}}$) to enforce a more exploitative search in later iterations.

\paragraph{3. Stochastic Region Selection.}
We employ a stochastic selection policy instead of a greedy one in favor of regions with maximal score. The selection probability $\pi_j(t)$ for region $R_j$ at iteration $t$ is defined via the Softmax function:
\begin{equation}
  \pi_j(t) = \frac{\exp(A_j(t))}{\sum_{k=1}^{K_t} \exp(A_k(t))}.
\end{equation}
We denote with $\mathbf{A} = (A_1, \dots, A_{K_t})$ the vector of scores, $\boldsymbol{\pi} = (\pi_1, \dots, \pi_{K_t})$ the vector of probability masses and with $\mathtt{Cat}(\boldsymbol{\pi})$ the Categorical distribution over $\mathcal{L}_t$, where we sample $k \leq K_t$ regions without repetition.
This stochastic formulation guarantees that every active region retains a strictly positive selection probability $\pi_j > 0$, thereby infinite visit count in the limit (consistent with Assumption~\ref{ass:exploration} in Section~\ref{sec:theory}), a necessary condition for almost-sure convergence to the Pareto set, which we prove in Section~\ref{sec:theory}.

\paragraph{4. LLM-based Generative Sampling.}
After sampling a subset $\mathcal{L}^{sampled}_{t} \subset \mathcal{L}_t$ of $k$ regions from $\mathtt{Cat}(\boldsymbol{\pi})$, we deploy a LLM as a local conditional sampler to generate $N$ new candidate points inside these sampled regions:
\begin{equation}
    \hat{x} \sim p_{\text{LLM}}(x \mid \mathcal{C}(R_j, \mathcal{D}_t)),
\end{equation}
where $\mathcal{C}$ denotes the prompt construction function. A prompt instance provided to the LLM has the following simplified instructions (see Appendix~\ref{appendix:prompt_design} for the full prompt templates): \textit{"Given the history of evaluations $\mathcal{D}_t$, generate $N$ new candidate points within the region $R_j$."}.
The use of LLMs in this setting is motivated from two main reasons. First, the LLM can be seen as flexible non-parametric generators, with a very rich prior (result of pre-training), that does not require assumptions on functions as, for instance, GPs do~\cite{Rasmussen2006Gaussian}. Second, in-context learning enables quick adaptation to diverse tasks without the need of retraining~\cite{brown2020language,krishnamurthy24}.

\paragraph{5. LLM-based Surrogate Modeling and Evaluation.}
At each iteration of MOHOLLM, we evaluate the true functions on a batch $b$ of new points. However, in step 4 we generated $k\cdot N$ candidates ($N$ candidates in $k$ regions) using the LLM. Typically, $k\cdot N > b$, so we need a strategy to select the subset of $b$ points. 
To this end, similarly as in step 4, we again utilize the LLM's capabilities to act as a performance predictor (surrogate model), generating predicted objective values $\hat{\mathbf{y}} \sim p_{\text{LLM}}(\mathbf{y} \mid \mathcal{C}(\mathcal{D}_t,\hat{x}))$ for the entire pool of generated candidates. From this synthetic pool, we then select a batch $B^*$ of size $b$ that maximizes the predicted HV indicator. Only this optimized batch is evaluated on the expensive true functions. Finally, we update the history of evaluations $\mathcal{D}_t$ with the new points and repeat the above steps until we reach the total evaluation budget $T$.

\paragraph{Search Process Visualization.} In Figure~\ref{fig:search_process_branin} we illustrate MOHOLLM's search on the 2D BraninCurrin (see Figure~\ref{fig:search_process_chankong} and \ref{fig:search_process_poloni} in Appendix~\ref{secapp:experiments} for Chankong-Haimes and Poloni, respectively). Early iterations (Trial 4) show coarse partitions encouraging global exploration, while later iterations (Trial 23) show fine-grained partitions focused on promising regions. The region scoring consistently balances exploration and exploitation by allocating resources to both sparse regions and dense Pareto-front regions simultaneously.

%% file: theory.tex
\section{Theoretical Analysis}
\label{sec:theory}


Let $\X \subset \R^d$ be a compact hyperrectangular domain and let $\X^\star \subset \X$ denote the Pareto set induced by the vector-valued objective
$\mathbf{F} = (f_1,\dots,f_M)$. Let $\mathcal{S}_t = \{x_1,\dots,x_t\}$ be the set of points sampled by MOHOLLM up to iteration~$t$.
To quantify convergence, we adopt the (directed) Hausdorff distance: $d_H(\X^\star, \mathcal{S}_t) = \sup_{x^\star \in \X^\star} \inf_{x \in \mathcal{S}_t} \|x - x^\star\|_2 $, a set distance metric commonly used in multi-objective optimization.
Our goal is to show that, under mild and standard regularity assumptions, which we adopt from optimistic and hierarchical optimization theory \cite{munos11,bubeck11}, and assumptions regarding the stochastic selection and LLM sampling, $d_H(\X^\star, \mathcal{S}_t) \to 0$ almost surely, implying that every Pareto-optimal point is approximated arbitrarily well by MOHOLLM's sampled points.

\begin{assumption}[Lipschitz Continuity]
\label{ass:lipschitz}
Each objective $f_i : \X \to \R$ is Lipschitz continuous. Thus, $\mathbf{F}$ is continuous and maps compact sets to compact sets.
\end{assumption}
This standard condition links spatial resolution to approximation error, bounding objective changes under small decision-space perturbations.

\begin{assumption}[Hierarchical Partitioning]
\label{ass:partition}
MOHOLLM maintains a hierarchical partition of $\X$ using a KD-tree. Let $R_t(x)$ denote a leaf region containing $x$ at iteration~$t$. For any infinite sequence of nested regions,
\(
R_{1} \supset R_{2} \supset \dots ,
\)
generated along a branch of the tree, we assume
$
\lim_{t \to \infty} \mathrm{diam}\!\left(R_{t}\right) = 0 .
$
\end{assumption}
This assumption guarantees that recursive refinement can localize arbitrarily small neighborhoods. It is satisfied by standard axis-aligned splitting rules and is a cornerstone of convergence proofs for hierarchical optimistic optimization \cite{munos11,bubeck11}.

\begin{assumption}[Sufficient Exploration]
\label{ass:exploration}
For any leaf region $R_k$ created at $t_0$, the selection probabilities $\pi_k(t)$ satisfy
$
\sum_{t=t_0}^{\infty} \pi_k(t) = \infty \quad \text{almost surely}.
$
\end{assumption}
This is a non-starvation condition ensuring that every region that persists in the tree is sampled infinitely often. In MOHOLLM, region scores are bounded through a sigmoid transformation, and the resulting Softmax normalization yields $\pi_k(t) = \Omega(1/K_t)$, where $K_t$ is the number of leaf nodes. If $K_t \le t$, e.g. in a fixed tree, the series diverges by comparison with the harmonic series.

\begin{assumption}[Eventual Refinement]
\label{ass:refinement}
Any leaf region $R$ selected infinitely often is split infinitely often, producing child regions of strictly smaller diameter.
\end{assumption}
This assumption prevents stagnation at a fixed spatial resolution, mirroring refinement guarantees used in optimistic tree search~\cite{kleinberg08,bubeck11}.

\begin{assumption}[LLM Sampling]
\label{ass:llm}
Let $p_{\mathrm{LLM}}(x \mid \mathcal{C}(R_j, \cdot))$ denote the conditional density of the LLM-generated proposal given region $R_j$. For any measurable ball $B \subset R$, there exists $\eta > 0$ such that
\[
\int_B P_{\mathrm{LLM}}(x \mid \mathcal{C}(R_j, \cdot)\,dx
\;\ge\;
\eta \, \frac{\mathrm{vol}(B)}{\mathrm{vol}(R_j)} .
\]
\end{assumption}

With this assumption, we ensure that the LLM assigns non-zero probability mass to any subregion $B$ inside the selected region $R_j$, hence avoiding complete mode collapse.

The aforementioned assumptions guarantee that MOHOLLM balances global exploration with local refinement, ensuring asymptotic coverage of the Pareto set without relying on parametric surrogates. Before proving such claim, we first establish an intermediate result regarding the spatial resolution of the partition around optimal points.

\begin{lemma}[The Shrinking Lemma]
\label{lem:shrinking}
For any $x^* \in \Xstar$, the diameter of the leaf node containing $x^*$ converges to zero almost surely, i.e. $\lim_{t \to \infty} \mathrm{diam}(R_t(x^*)) = 0 \quad \text{a.s.}$
\end{lemma}



\begin{proof}
We use contradiction for the proof. Assume that the diameter does not converge to zero. Given the hierarchical nature of the partition, this implies that after some finite time $T_0$, the leaf $\bar{R} = R_{T_0}(x^*)$ is never split again. 
Since $\bar{R}$ remains a leaf for all $t > T_0$, Assumption \ref{ass:exploration} ensures that its cumulative selection probability diverges: $\sum_{t=T_0}^{\infty} \pi_{\bar{R}}(t) = \infty$.
By the second Borel-Cantelli lemma, this series' divergence implies that $\bar{R}$ is selected infinitely often. However, Assumption \ref{ass:refinement} requires that any region selected infinitely often must eventually split. This contradicts the assumption that $\bar{R}$ remains fixed after some time. Consequently, the sequence of regions containing $x^*$ must undergo infinite refinement, driving the diameter to zero (Assumption \ref{ass:partition}).
\end{proof}

\begin{figure*}[ht]
    \centering
    \includegraphics[width=0.99\linewidth]{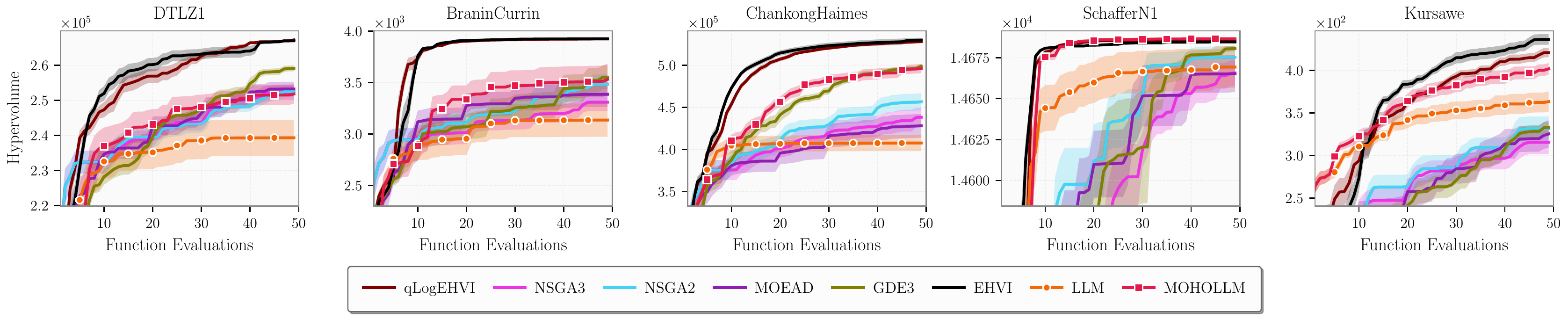}
    \caption{Hypervolume (HV) over function evaluations on synthetic benchmarks (DTLZ1, Branin-Currin, Chankong–Haimes, SchafferN1 and Kursawe). MOHOLLM and the LLM baseline curves are marked to highlight them further.}
    \label{fig:hv_synth}
\end{figure*}

\begin{figure*}[t]
\centering
\begin{minipage}[t]{0.62\linewidth}
    \centering
    \includegraphics[width=\linewidth]{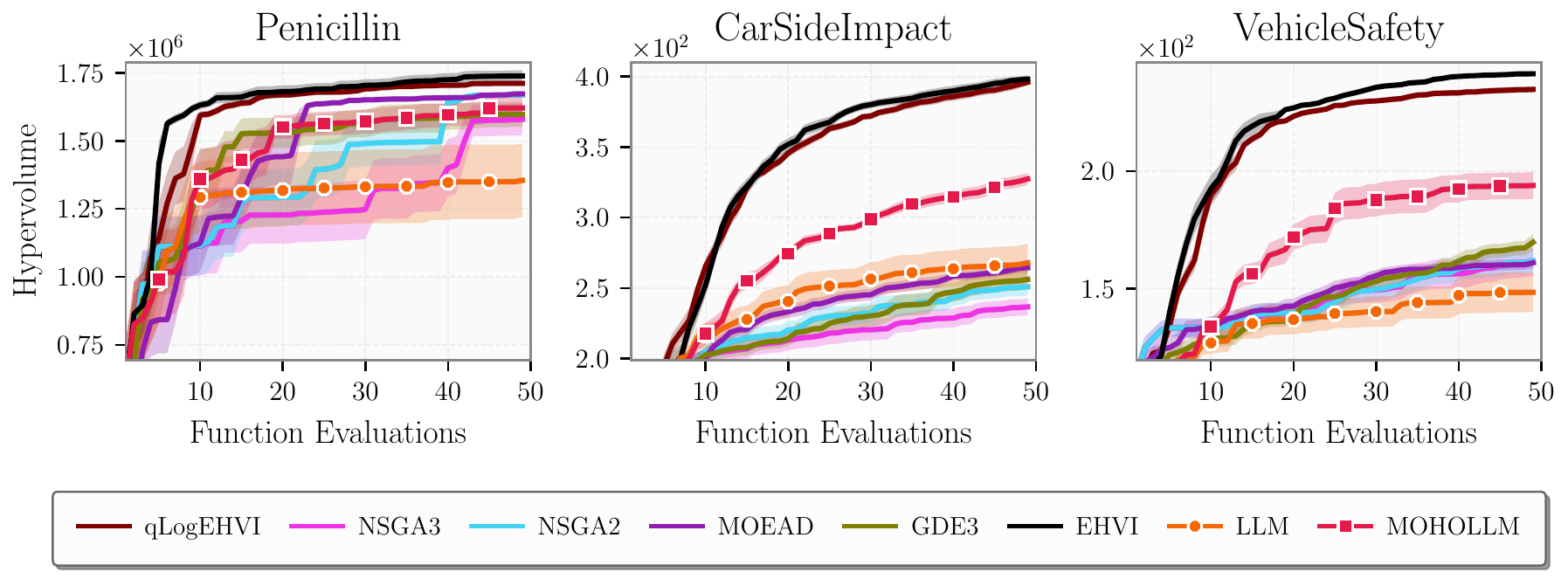}
    \caption{Hypervolume (HV) of MOHOLLM and baselines over function evaluations on real-world benchmarks (Penicillin, Vehicle Safety, and Car Side Impact).}
    \label{fig:hv_real}
\end{minipage}\hfill
\begin{minipage}[t]{0.36\linewidth}
    \centering
    \includegraphics[width=\linewidth]{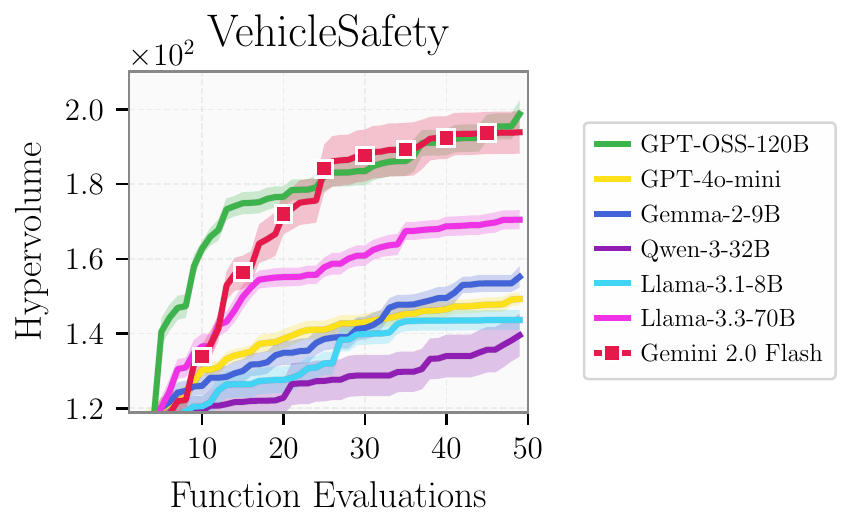}
    \caption{HV trajectories of MOHOLLM with different LLM surrogates and samplers.}
    \label{fig:hv_real_other}
\end{minipage}
\end{figure*}

\begin{theorem}[Almost-Sure Pareto Consistency]
\label{thm:consistency}
Under Assumptions \ref{ass:lipschitz}--\ref{ass:refinement}, the sample set $\mathcal{S}_t$ converges to the Pareto set $\Xstar$ in the Hausdorff distance almost surely:
\[
\lim_{t \to \infty} d_H(\Xstar, \mathcal{S}_t) = 0 \quad \text{a.s.}
\]
\end{theorem}

\begin{proof}
Since $\Xstar$ is compact, it suffices to show that for any $\varepsilon > 0$ and any target $x^* \in \Xstar$, the algorithm eventually places a sample inside the open ball $B(x^*, \varepsilon)$.
Fix $x^* \in \Xstar$ and $\varepsilon > 0$. By Lemma \ref{lem:shrinking}, the diameter of the containing cell $R_t(x^*)$ converges to 0 almost surely. Thus, there exists a finite time $T_{\text{loc}}$ such that for all $t > T_{\text{loc}}$, $R_t(x^*) \subset B(x^*, \varepsilon/2)$.

For $t > T_{\text{loc}}$, any sample generated within $R_t(x^*)$ necessarily lies within $B(x^*, \varepsilon)$. From the proof of Lemma \ref{lem:shrinking}, we know that the sequence of regions $\{R_t(x^*)\}_{t}$ is selected infinitely often. Let $\{t_k\}$ be the sequence of iterations where such a selection occurs. By Assumption \ref{ass:llm}, conditioned on selecting $R_{t_k}(x^*)$, the LLM generates a valid point $\hat{x}_{new} \in R_{t_k}(x^*)$ with probability at least $\eta > 0$. The probability that no sample ever lands in $B(x^*, \varepsilon)$ is bounded by $\lim_{K \to \infty} (1 - \eta)^K = 0$. Thus, $B(x^*, \varepsilon) \cap \mathcal{S}_t \neq \emptyset$ almost surely as $t \to \infty$.

Finally, by compactness, $\Xstar$ admits a finite cover of $\varepsilon/2$-balls centered at $\{z_1, \dots, z_N\}$. Applying the argument above to each $z_i$ simultaneously, there exists a time $T$ such that $\mathcal{S}_T$ intersects every ball in the cover. This implies $\sup_{x^* \in \Xstar} \inf_{x \in \mathcal{S}_T} \|x - x^*\| < \varepsilon$. Since $\varepsilon$ is arbitrary, the Hausdorff distance converges to 0 almost surely.
\end{proof}





\begin{corollary}[Hypervolume Consistency]
\label{cor:hv_consistency}
The hypervolume of the objective vectors $\mathbf{F}(\mathcal{S}_t)$ converges to the hypervolume of the true Pareto Front $\mathbf{F}(\Xstar)$ almost surely.
\end{corollary}

\begin{proof}
Theorem~\ref{thm:consistency} establishes that $\lim_{t \to \infty} d_H(\Xstar, \mathcal{S}_t) = 0$. Since $\mathbf{F}$ is Lipschitz continuous (Assumption~\ref{ass:lipschitz}), this implies that the image $\mathbf{F}(\mathcal{S}_t)$ converges to the true Pareto front $\mathbf{F}(\Xstar)$ in the objective space. Since the HV is Lipschitz continuous w.r.t. $d_H$ on compact domains, the convergence of $\mathbf{F}(\mathcal{S}_t)$ guarantees the convergence of the HV too.
\end{proof}
